%% file: main.tex
\documentclass[letterpaper,11pt]{article}

\input{macros}

\usepackage{fontawesome}
\usepackage{xcolor}
\usepackage{cite}
\usepackage{url}
\usepackage{xspace}

\newcommand{\ps}{p^*}
\newcommand{\pt}{\ptilde}
\newcommand{\ys}{y^*}
\newcommand{\yt}{\tilde{y}}
\newcommand{\calH}{\mathcal{H}}
\newcommand{\cYh}{\mathcal{\hat{Y}}}
\renewcommand{\yhat}{\hat{y}}
\newcommand{\Ber}{\mathrm{Ber}}
\newcommand{\ind}{\mathbf{1}}
\newcommand{\lmax}{\ell_{\max}}

\newcommand{\ft}{\tilde{f}}

\newcommand{\unif}{\mathrm{Unif}}

\title{Making Decisions under Outcome Performativity} 
\author{Michael P. Kim \\ University of California, Berkeley\\ mpkim@berkeley.edu \and Juan C. Perdomo\\ University of California, Berkeley\\jcperdomo@berkeley.edu}
\date{\today}
\begin{document}


\maketitle

\pagenumbering{gobble}

\blfootnote{Authors listed alphabetically.}

\begin{abstract}
Decision-makers often act in response to data-driven predictions, with the goal of achieving favorable outcomes.
In such settings, predictions don’t passively forecast the future; instead, predictions actively shape the distribution of outcomes they are meant to predict.
This \emph{performative prediction} setting \cite{performative} raises new challenges for learning ``optimal'' decision rules.
In particular, existing solution concepts do not address the apparent tension between the goals of \emph{forecasting} outcomes accurately and \emph{steering} individuals to achieve desirable outcomes.

To contend with this concern, we introduce a new optimality concept---\emph{performative omniprediction}---adapted from the supervised (non-performative) learning setting \cite{omni}.
A performative omnipredictor is a single predictor that simultaneously encodes the optimal decision rule with respect to many possibly-competing objectives.
Our main result demonstrates that efficient performative omnipredictors exist, under a natural restriction of performative prediction, which we call \emph{outcome performativity}.
On a technical level, our results follow by carefully generalizing the notion of outcome indistinguishability \cite{oi,gopalan2022loss} to the outcome performative setting.
From an appropriate notion of Performative OI, we recover many consequences known to hold in the supervised setting, such as omniprediction and universal adaptability \cite{ua}.
\end{abstract}

\clearpage

\pagenumbering{arabic}
\input{intro}



\input{performative-omni}

\input{universality.tex}
\input{learning.tex}


\input{calibration}

\clearpage

\subsection*{Acknowledgments}
The authors thank Parikshit Gopalan, Moritz Hardt, Celestine Mendler-Dunner, Omer Reingold, and Tijana Zrnic for helpful discussions throughout the development of the project.
\textbf{MPK} is supported by the Miller Institute for Basic Research in Science and, in part, by the Simons Collaboration on Algorithmic Fairness.

\bibliographystyle{alpha}
\bibliography{refs}







\end{document}

%% file: macros.tex
\usepackage{diagbox}
\usepackage{makecell}
\usepackage{booktabs}
\usepackage{tikz,tikz-3dplot}
\usepackage{amsmath}
\usepackage{bbm}
\usepackage{amsfonts}
\usepackage{amssymb}
\usepackage{amsthm}
\usepackage{url,ifthen}
\usepackage[shortlabels]{enumitem}
\usepackage{srcltx}
\usepackage{dsfont}
\usepackage{multirow}
\usepackage{boxedminipage}
\usepackage[margin=1.1in]{geometry}
\usepackage{nicefrac}
\usepackage{xspace}
\usepackage{graphicx}
\usepackage{color}
\usepackage{subfigure}
\usepackage{colortbl}
\usepackage{setspace}
\usepackage{pgfplots}
\pgfplotsset{compat=1.12}
\usepackage{algorithm,algorithmic}
\usepackage[algo2e]{algorithm2e}

\usepackage{xcolor}
\definecolor{DarkGreen}{rgb}{0.1,0.5,0.1}
\definecolor{DarkRed}{rgb}{0.5,0.1,0.1}
\definecolor{DarkBlue}{rgb}{0.1,0.1,0.5}
\definecolor{Gray}{rgb}{0.2,0.2,0.2}

\definecolor{c1}{RGB}{38, 70, 83}
\definecolor{c2}{RGB}{42, 157, 143}
\definecolor{c3}{RGB}{233, 196, 106}
\definecolor{c5}{RGB}{231, 111, 81}
\definecolor{c4}{RGB}{244, 162, 97}

\definecolor{c1}{RGB}{38, 70, 83}
\definecolor{c2}{RGB}{42, 157, 143}
\definecolor{c3}{RGB}{233, 196, 106}
\definecolor{c5}{RGB}{231, 111, 81}
\definecolor{c4}{RGB}{244, 162, 97}
\newcommand\blfootnote[1]{%
  \begingroup
  \renewcommand\thefootnote{}\footnote{#1}%
  \addtocounter{footnote}{-1}%
  \endgroup
}

\usepackage{listings}
\lstdefinestyle{mystyle}{
    commentstyle=\color{DarkBlue},
    keywordstyle=\color{DarkRed},
    numberstyle=\tiny\color{Gray},
    stringstyle=\color{DarkGreen},
    basicstyle=\footnotesize,
    breakatwhitespace=false,         
    breaklines=true,                 
    captionpos=b,                    
    keepspaces=true,                 
    numbers=left,                    
    numbersep=5pt,                  
    showspaces=false,                
    showstringspaces=false,
    showtabs=false,                  
    tabsize=2
}
\usepackage[small]{caption}
\usepackage[pdftex]{hyperref}
\hypersetup{
    unicode=false,          
    pdftoolbar=true,        
    pdfmenubar=true,        
    pdffitwindow=false,      
    pdfnewwindow=true,      
    colorlinks=true,       
    linkcolor=DarkBlue,          
    citecolor=DarkGreen,        
    filecolor=DarkRed,      
    urlcolor=DarkBlue,          
    pdftitle={},
    pdfauthor={},
}

\def\draft{1}

\def\submit{0}

\ifnum\draft=1 
    \def\ShowAuthNotes{1}
\else
    \def\ShowAuthNotes{0}
\fi

\ifnum\submit=1
\newcommand{\forsubmit}[1]{#1}
\newcommand{\forreals}[1]{}
\else
\newcommand{\forreals}[1]{#1}
\newcommand{\forsubmit}[1]{}
\fi

\ifnum\ShowAuthNotes=1
\newcommand{\authnote}[2]{{ \footnotesize \bf{\color{DarkRed}[#1's Note:
{\color{DarkBlue}#2}]}}}
\else
\newcommand{\authnote}[2]{}
\fi

\newtheorem{theorem}{Theorem}[section]

\newtheorem{lemma}[theorem]{Lemma}
\newtheorem{corollary}[theorem]{Corollary}
\newtheorem{proposition}[theorem]{Proposition}
\newtheorem{fact}[theorem]{Fact}

\newtheorem{definition}[theorem]{Definition}

\newtheorem*{definition*}{Definition}
\newtheorem*{proposition*}{Proposition}
\newtheorem{result}{Theorem}

\theoremstyle{definition}
\newtheorem*{example*}{Example}

\newtheoremstyle{example_contd}
{\topsep} {\topsep}%
{}
{}
{\bfseries}
{.}
{1em}
{\thmname{#1} \thmnumber{ #2}\thmnote{#3} (continued)}

\theoremstyle{example_contd}

\newcommand{\chapterref}[1]{\hyperref[ch:#1]{Chapter~\ref{ch:#1}}}

\newcommand{\claimref}[1]{\hyperref[claim:#1]{Claim~\ref{claim:#1}}}
\newcommand{\corollarylabel}[1]{\label{cor:#1}}
\newcommand{\corollaryref}[1]{\hyperref[cor:#1]{Corollary~\ref{cor:#1}}}
\newcommand{\definitionlabel}[1]{\label{def:#1}}
\newcommand{\definitionref}[1]{\hyperref[def:#1]{Definition~\ref{def:#1}}}
\newcommand{\equationlabel}[1]{\label{eq:#1}}
\newcommand{\equationref}[1]{\hyperref[eq:#1]{Equation~\ref{eq:#1}}}

\newcommand{\factref}[1]{\hyperref[fact:#1]{Fact~\ref{fact:#1}}}
\newcommand{\figurelabel}[1]{\label{fig:#1}}
\newcommand{\figureref}[1]{\hyperref[fig:#1]{Figure~\ref{fig:#1}}}

\newcommand{\tableref}[1]{\hyperref[tab:#1]{Table~\ref{tab:#1}}}

\newcommand{\itemref}[1]{\hyperref[item:#1]{Item~(\ref{item:#1})}}
\newcommand{\lemmalabel}[1]{\label{lem:#1}}
\newcommand{\lemmaref}[1]{\hyperref[lem:#1]{Lemma~\ref{lem:#1}}}
\newcommand{\proplabel}[1]{\label{prop:#1}}
\newcommand{\propref}[1]{\hyperref[prop:#1]{Proposition~\ref{prop:#1}}}
\newcommand{\propositionlabel}[1]{\label{prop:#1}}
\newcommand{\propositionref}[1]{\hyperref[prop:#1]{Proposition~\ref{prop:#1}}}

\newcommand{\remarkref}[1]{\hyperref[rem:#1]{Remark~\ref{rem:#1}}}
\newcommand{\sectionlabel}[1]{\label{sec:#1}}
\newcommand{\sectionref}[1]{\hyperref[sec:#1]{Section~\ref{sec:#1}}}
\newcommand{\theoremlabel}[1]{\label{thm:#1}}
\newcommand{\theoremref}[1]{\hyperref[thm:#1]{Theorem~\ref{thm:#1}}}

\usepackage[utf8]{inputenc}
\usepackage[T1]{fontenc}
\usepackage{kpfonts}
\usepackage{microtype}

\newcommand{\Esymb}{\mathbb{E}}
\newcommand{\Psymb}{\mathbb{P}}

\DeclareMathOperator*{\E}{\Esymb}

\DeclareMathOperator*{\ProbOp}{\Psymb r}
\renewcommand{\Pr}{\ProbOp}

\renewcommand{\hat}{\widehat}

\newcommand{\cA}{{\cal A}}

\newcommand{\cD}{{\cal D}}

\newcommand{\cH}{{\cal H}}

\newcommand{\cL}{{\cal L}}

\newcommand{\cO}{{\cal O}}

\newcommand{\cW}{{\cal W}}
\newcommand{\cX}{{\cal X}}
\newcommand{\cY}{{\cal Y}}

\newcommand{\defeq}{\stackrel{\small \mathrm{def}}{=}}
\renewcommand{\leq}{\leqslant}
\renewcommand{\le}{\leqslant}
\renewcommand{\geq}{\geqslant}
\renewcommand{\ge}{\geqslant}

\newcommand{\set}[1]{\{#1\}}

\newcommand{\card}[1]{\lvert#1\rvert}

\newcommand{\norm}[1]{\lVert#1\rVert_2}

\newcommand{\R}{\mathbb{R}}

\renewcommand{\D}{\mathcal D}

\renewcommand{\qedsymbol}{{$\blacksquare$}}
\usepackage{bm}

\newcommand{\ignore}[1]{}

\newcommand{\sign}{\mathrm{sign}}
\newcommand{\polylog}{{\rm polylog}}
\newcommand{\poly}{{\rm poly}}

\newcommand{\supp}{\mathrm{supp}}
\DeclareMathOperator*{\argmin}{arg\,min}

\renewcommand{\epsilon}{\varepsilon}

\newcommand{\eps}{\epsilon}

\newcommand{\Unif}{\mathrm{Unif}}

\newcommand{\remove}[1]{}

\newcommand{\calY}{\mathcal{Y}}
\newcommand{\calYhat}{\hat{\calY}}
\newcommand{\calD}{\mathcal{D}}

\newcommand{\yhat}{\hat{y}}
\newcommand{\qt}{\tilde{q}}
\newcommand{\ptilde}{\tilde{p}}

\newcommand{\cscalg}{\cA_{\mathrm{csc}}}

\newcommand{\omax}{\omega_{\max}}

\newcommand{\Hcal}{\mathcal{H}}
\newcommand{\Xcal}{\cX}
\newcommand{\Ych}{\cYh}
\newcommand{\1}{\ind}

\newcommand{\Wcal}{\mathcal{W}}

\newcommand{\rct}{\D_\mathrm{rct}}

%% file: intro.tex
\section{Introduction}

Data-driven predictions inform policy decisions that directly impact individuals.
Proponents argue that by understanding patterns from the past, decisions can be optimized to improve future outcomes, to the benefit of individuals and institutions \cite{kleinberg2015prediction}.
In the US educational system, for instance, early warning systems (EWS) have become a key tool used by states to combat low graduation rates \cite{balfanz2019early,survey}. 
The rationale for using such systems is clear.  Given a predictor that, for each student, estimates the likelihood of graduation, school districts can identify high-risk students at a young age, directing resources to improve individuals' outcomes, and in turn, the districts' graduation rates.
Despite compelling arguments, reliably predicting life outcomes remains a largely-unsolved problem in machine learning.

A key challenge in utilizing predictions to inform decisions is that, often, predictions influence the outcomes they're meant to forecast.
In the education example above, districts consider predictions of graduation with the \emph{intention} of effecting graduation outcomes.
In this situation---where predictions determine interventions, which influence outcomes---accuracy can be a paradoxical notion.
If a predictor correctly identifies high risk individuals as likely to suffer negative outcomes, after successful interventions, the individuals' outcomes will be positive and the initial predictions will appear inaccurate.
To apply data-driven tools effectively, decision-makers must resolve an apparent tension between the objectives of \emph{forecasting} individuals' outcomes reliably and \emph{steering} individuals to achieve better outcomes.


Recent work of \cite{performative} introduced \emph{performative prediction} to contend with the fact that predictions not only forecast, but also shape the world.
Informally, a prediction problem is performative if the act of prediction influences the distribution on individual-outcome pairs.
From early warning systems, to online content recommendations, to public health advisories: across many contexts, individuals respond to predictions in a manner that changes the likelihood of possible outcomes (successful graduation, increased click rate, or decreased disease caseload).




In their original work on the subject, \cite{performative} frame the goal of performative prediction through loss minimization.
In this framing, the ultimate goal is to learn a \emph{performatively optimal} decision rule.
A decision rule $h_\mathrm{po}$ is performatively optimal if it achieves the minimal expected loss (within some class of decision rules $\Hcal$) over the distribution that it induces,
\begin{gather}
\equationlabel{eq:po}
    h_\mathrm{po} \in \argmin_{h \in \Hcal} \E_{(x,y) \sim \D(h)}[\ell(x,h(x),y)].
\end{gather}
Here, $\D(h)$ is the distribution over $(x,y)$ pairs observed as response to deploying $h$.

For generality's sake, performative prediction makes minimal restrictions on how the distribution may respond to a chosen decision rule.
In particular, the choice to deploy a hypothesis $h$, may change the joint distribution $(x,y) \sim \D(h)$ over individual-outcome pairs, essentially arbitrarily.\footnote{\cite{performative}  assume only a Lipschitzness condition, where similar hypotheses $h$ and $h'$ give rise to similar distributions $\D(h)$ and $\D(h')$, measured in Wasserstein (earth mover's) distance.}
This generality enables us to write a broad range of prediction problems---including supervised learning \cite{shalev2014understanding}, strategic classification \cite{hardt2016strategic}, and causal inference \cite{miller2020strategic}---as special cases of performative prediction.
In all, \cite{performative} establishes a powerful framework for reasoning about settings where the distribution of examples responds to the predictions.


While powerful, the framework has two noticeable limitations.
First, achieving performative optimality is hard.
Without any assumptions on the distributional response $\D(\cdot)$, achieving performative optimality requires exhaustive search over the hypothesis class $\Hcal$.
Furthermore, even under strong structural assumptions on the distributional response and choice of loss $\ell$, it is known that convex optimization does not suffice to achieve optimality \cite{performative, miller2021outside}.
Stated another way: the generality of performative prediction does not come for free.
To date, all existing methods for performative optimality require strong specification assumptions on the outcome distribution and distributional response.

The second limitation arises due to formulating performative prediction as a loss minimization problem:  the loss $\ell$ is fixed, once and for all.
In performative prediction, different losses can encode drastically different objectives:  losses are used not only to promote accuracy of predictions, but also to encourage favorable outcome distributions.
Consider a loss designed for accurate forecasting, e.g., the squared error $(\yhat - y)^2$.
In this case, the optimal decision rule will prioritize accuracy without regard for the ``quality'' of the outcome distribution.
On the other hand, consider a loss designed to steer towards positive outcomes, $1-y$.
Here, there is no notion of accuracy (the loss ignores the prediction $\yhat$), but instead, the objective is to nudge the distribution of outcomes towards $y = 1$.

Encoding the decision-making objective through a single loss function forces the learner to choose the ``correct'' objective at train time.
Downstream decision-makers, however, may reasonably want to explore different objectives according to their own sense of ``optimality''.
In the existing formulations for performative prediction, exploring different losses requires re-training from scratch.
In this work, we investigate an alternative formulation that enables decision-makers to efficiently explore optimal decision rules under many different objectives.

\subsection{Decision-Making under Outcome Performativity}

To begin, we introduce a special case of the performative prediction setting, which we call \emph{outcome performativity}.
Outcome performativity focuses on the effects of local decisions on individuals' outcomes, rather than the effect of broader policy on the distribution of individuals.
For instance, our example of graduation prediction is modeled well by outcome performativity.
For a given a student, the EWS prediction they receive affects their future graduation outcome, but does not influence their demographic features or historical test scores.
In other words, we narrow our attention to the performative effects of decisions $h(x)$ on the conditional distribution over outcomes $y$, rather than the effects of the decision rule $h$ on the distribution as a whole $\D(h)$.
This reframing of performativity still captures many important decision-making problems, but gives us additional structure to address some of the limitations in the original formulation.


On a technical level, outcome performativity imagines a data generating process over triples $(x,\yhat, \ys)$ where $x \sim \D$ is sampled from a \emph{static} distribution over inputs, then a prediction or decision $\yhat \in \Ych$ is selected (possibly as a function of $x$), and finally the true outcome $\ys \in \cY$ is sampled conditioned on $x$ and $\yhat$.
We focus on binary outcomes $\cY = \set{0,1}$.\footnote{In general, outcome performativity could be defined for larger outcome domains.  Handling such domains is possible, but technical.  We restrict our attention to binary outcomes to focus on the novel conceptual issues.}
In this setting, the outcome performativity assumption posits the existence of an underlying probability function,
\begin{align*}
\ps:\Xcal \times \Ych \to [0,1],	
\end{align*}
where for a given individual $x \in \Xcal$ and decision $\yhat \in \Ych$, the true outcome $\ys$ is sampled as a Bernoulli with parameter $\ps(x,\yhat)$.
We refer to the true outcome distribution $\ps$ as \emph{Nature}.

By asserting a fixed ``ground truth'' probability function, the outcome performativity framework does not allow for arbitrary distributional responses and limits the generality of the approach. For instance, outcome performativity does not capture strategic classification.
But importantly, by refining the model of performativity, there is hope that we may sidestep the hardness results for learning optimal performative predictors.

\begin{figure}
\begin{center}
\begin{tikzpicture}[scale=0.2]
\tikzstyle{every node}+=[inner sep=0pt]
\draw [black] (30.7,-27.5) circle (3);
\draw (30.7,-27.5) node {$x$};
\draw [black] (38.7,-18.7) circle (3);
\draw (38.7,-18.7) node {$\yhat$};
\draw [black] (46.2,-27.5) circle (3);
\draw (46.2,-27.5) node {$\ys$};
\draw [black] (32.72,-25.28) -- (36.68,-20.92);
\fill [black] (36.68,-20.92) -- (35.77,-21.18) -- (36.51,-21.85);
\draw [black] (33.7,-27.5) -- (43.2,-27.5);
\fill [black] (43.2,-27.5) -- (42.4,-27) -- (42.4,-28);
\draw [black] (40.6,-21.3) -- (44.19,-25.27);
\fill [black] (44.19,-25.27) -- (44.02,-24.34) -- (43.28,-25.02);
\end{tikzpicture}
\end{center}
\caption{Causal graphical representation of the outcome performativity data generating process.}
\end{figure}
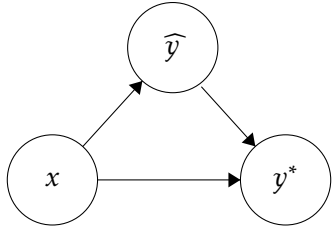


\newcommand{\fs}{f^*}
\paragraph{Performative Omniprediction.}
We begin by observing that under outcome performativity, the true probability function $\ps$ suggests an optimal decision rule $\fs_\ell:\Xcal \to \Ych$ for any loss $\ell$.
In our setting, $\ps$ governs the outcome distribution, so given an input $x \in \Xcal$, the optimal decision $\fs_\ell(x)$ is determined by a simple, univariate optimization procedure over a discrete set $\Ych$:
\begin{gather}
\equationlabel{eq:opt_dr}
    \fs_\ell(x) \in \argmin_{\yhat \in \Ych}\E_{\ys \sim \ps(x,\yhat)}[\ell(x,\yhat,\ys)].
\end{gather}
Note that the decision rule $\fs_\ell(x)$ minimizes the loss pointwise for $x \in \Xcal$. Consequently, averaging over any static, feature distribution $\D$, the decision rule $\fs_\ell$ is performative optimal
for \emph{any} hypothesis class $\Hcal$, loss $\ell$, and marginal distribution $\cD$:
\begin{gather*}
    \E_{\substack{x \sim \D\\\ys \sim \ps(x,\fs_\ell(x))}}[\ell(x,\fs_\ell(x),\ys)] \le \min_{h \in \Hcal} \E_{\substack{x \sim \D\\\ys \sim \ps(x,h(x))}}[\ell(x,h(x),\ys)].
\end{gather*}
While the existence of $\ps$ implies the existence of optimal decision rules under outcome performativity, we make no assumptions about the learnability of $\ps$.
In general, the function $\ps$ may be arbitrarily complex, so learning (or even representing!) $\ps$ may be infeasible, both computationally and statistically.
Still, the above analysis reveals the power of modeling the probability function $\ps:\Xcal \times \Ych \to [0,1]$.
The optimal probability function $\ps$ encodes the optimal decision rule $\fs_\ell$ \emph{for every loss function } $\ell$.
This perspective raises a concrete technical question:  short of learning $\ps$, can we learn a probability function $\pt:\Xcal \times \Ych \to [0,1]$ that suggests an optimal decision rule, via simple post-processing, for many different objectives?

Recent work of \cite{omni} studied the analogous question in the context of supervised learning (without performativity), formalizing a solution concept which they call \emph{omniprediction}.
Intuitively, an omnipredictor is a single probability function $\pt$ that suggests an optimal decision rule for many different loss functions $\cL$.
The work of \cite{omni} and follow-up work of \cite{gopalan2022loss} demonstrate---rather surprisingly---that omniprediction in supervised learning is broadly a feasible concept.
For a variety of choices of loss classes $\cL$ (e.g., Lipschitz losses or convex losses), it is possible to learn an efficient predictor $\pt$ that gives optimal decisions for any loss $\ell \in \cL$.



In this work, we generalize omniprediction to the outcome performative setting.
As a solution concept, \emph{performative omniprediction} directly addresses the limiting assumption in performative prediction that the loss $\ell$ is known and fixed.
Given a performative omnipredictor, a decision-maker can explore the consequences of optimizing for different losses, balancing the desire for forecasting and steering, as they see fit.
Technically, given a predictor $\pt$, we define $\ft_\ell:\Xcal \to \Ych$ to be the optimal decision rule,
 that acts as if outcomes are governed by $\pt$.
\begin{gather*}
    \ft_\ell(x) \in \argmin_{\yhat \in \Ych}\E_{\yt \sim \pt(x,\yhat)}[\ell(x,\yhat,\yt)]
\end{gather*}
We emphasize that, for any loss $\ell$, the decision rule $\ft_\ell(x)$ is an efficient post-processing of the predictions given by $\pt(x,\yhat)$ for $\yhat \in \Ych$.
A performative omnipredictor is a model of nature $\pt:\cX \times \cYh \rightarrow [0,1]$ that induces a corresponding decision rule $\ft_\ell$ that is performatively optimal over a collection of losses $\ell \in \cL$.
\begin{definition*}[Performative Omnipredictor]
For a collection of loss functions $\cL$, hypothesis class $\Hcal$, and $\eps \ge 0$, a predictor $\pt:\Xcal \times \Ych \to [0,1]$ is an $(\cL,\Hcal,\eps)$-performative omnipredictor  for an input distribution $\cD$ if for every $\ell \in \cL$, the decision rule $\ft_\ell$ is $\eps$-performative optimal over $\Hcal$.
\begin{gather}
\equationlabel{eq:omni_guarantee}
    \E_{\substack{x \sim \D\\\ys \sim \ps(x,\ft_\ell(x)))}}[\ell(x,\ft_\ell(x)), \ys)] \le \argmin_{h \in \Hcal} \E_{\substack{x \sim \D\\\ys \sim \ps(x,h(x))}}[\ell(x, h(x),\ys)] + \eps
\end{gather}
\end{definition*}
While an intriguing prospect, omniprediction is particularly ambitious in the performative world.
Whereas most supervised learning losses have the same moral goal (to accurately forecast the outcome), losses in the performative world can encode entirely contradictory objectives.
For instance, we can define a pair of losses $\ell_0$ and $\ell_1$ that reward decisions that steer outcomes to be $0$ and $1$, respectively.
A performative omnipredictor must contend with these contradictions, providing optimal decision rules under performative effects.


Concretely, under outcome performativity, there is a certain circularity in naively determining the optimal decision $\ft(x)$ from a prediction $\pt(x,\yhat)$.
Choosing an ``optimal'' decision $\ft(x)$ causes a shift in the distribution on the outcome $\ys \sim \ps(x,\ft(x))$, which may imply a different ``optimal'' decision, which seems to lead to a continuing cycle of dependency.
In this way, any performative omnipredictor $\pt$ must encode the optimal decision rule $\ft_\ell$ for each $\ell \in \cL$, \emph{anticipating the shift} induced by the choice of $\ft_\ell$.
In this work, we ask whether---despite this key challenge---efficient performative omnipredictors exist, and if so, can we learn them?


\subsection{Our Contributions}

Our first contributions are conceptual, introducing the outcome performativity setting and the notion of performative omnipredictors.
As an abstraction, outcome performativity strikes a balance 
with enough generality to model many real-world phenomena and enough structure to give effective solutions.
For settings where the distributional response occurs predominantly as outcome performativity, the framework is well-scoped to contend with the challenges of performative prediction.
In particular, performative omnipredictors provide an effective solution concept to address the tension between different objectives under performativity.

With these conceptual contributions in place, we turn to the feasibility of omniprediction under outcome performativity.
While outcome performativity introduces a number of new challenges, we show how to apply many techniques established for omniprediction in the supervised learning setting to recover analogous guarantees under performativity.
On a technical level, we follow the \emph{Loss Outcome Indistinguishability} approach of \cite{gopalan2022loss}, demonstrating how---with the right conceptual framing---arguments for the existence of supervised omnipredictors can be translated into guarantees for the performative setting.

\paragraph{Efficient Performative Omnipredictors Exist.}
Our first technical contribution demonstrates existence of efficient performative omnipredictors.
We prove that for any class of losses $\cL$ and any hypothesis class $\Hcal$, there exists a performative omnipredictor $\pt$ of complexity that scales polynomially with the complexity of computing the losses and hypotheses.
\begin{result}
\label{result:omni}
Suppose $\D$ is a fixed distribution over $\Xcal$.
Let $\cL \subseteq \set{\ell:\Xcal\times \Ych \times \cY \to [0,1]}$ be a set of bounded loss functions, and let $\cH \subseteq \set{h:\Xcal \to \Ych}$ be a hypothesis class of decision rules. 
If the functions in $\cL$ and $\cH$ can be computed by circuits of size $s$, then there exists a $(\cL,\Hcal,\eps)$-performative omnipredictor of circuit complexity $\poly(s, |\cYh |) / \eps^2$.
\end{result}
Importantly, this result holds \emph{for any} class of bounded losses.
The collection $\cL$ may include losses for forecasting and steering, or may include losses that steer towards different outcomes.
Still, the predictor $\pt$ will encode a performative optimal decision rule for each such loss $\ell \in \cL$.
Furthermore, the complexity of this predictor scales gracefully with the complexity of the losses and hypotheses and the available decisions, \emph{independent of the complexity of Nature} $\ps$.
Even if the true probability function $\ps$ is intractably-complex, there exists a simple function $\pt$ that mimics the omniprediction behavior, provided the losses and hypotheses are sufficiently simple.


\paragraph{Learning Performative Omnipredictors Reduces to Supervised Learning.}

In fact, the proof of existence is constructive.
We establish the feasibility of performative omnipredictions by devising a boosting-style learning algorithm, inspired by the original algorithm for learning (non-performative) omnipredictors  \cite{hkrr,oi,omni,gopalan2022loss}.
As in the supervised case, we show that learning omnipredictors reduces to an \emph{auditing} task.
Despite the fact that in performative prediction, different decision rules induce different distributions, we show that given appropriately randomized data, this auditing task can be solved using only \emph{supervised learning} primitives implementable in finite samples.
That is, under outcome performativity, there is a surprising reduction from the task of learning optimal performative predictors to the task of non-performative supervised learning.

Formally, we assume that the learner has access to a collection of data triples $(x,\yhat,y) \sim \D_\mathrm{rct}$ where inputs are sampled from the data distribution $x \sim \D$, decisions $\yhat$ are assigned uniformly at random, and the outcome $\ys \sim \ps(x,\yhat)$ is sampled from Nature, for the given individual and randomly-assigned decision.
Given an efficiently bounded number of samples access from this distribution, we show how to learn performative omnipredictors assuming access to a supervised learner for the hypothesis class $\Hcal$.
We formalize this learning assumption in terms of cost-sensitive classification \cite{elkan2001foundations}.

\begin{result}[Informal]
\label{result:learning}
Assume sample access to $\D_\mathrm{rct}$  and
suppose that $\cA$ is a cost-sensitive learning algorithm for the hypothesis class $\Hcal$.
There is a polynomial-time algorithm, that, for any set of bounded losses $\cL$, returns a $(\cL, \cH, \epsilon)$-performative omnipredictor using at most $\poly(1/\epsilon,\; |\cYh|,\;\log |\cH|,\; \log|\cL|)$ many samples from $\D_\mathrm{rct}$ while also making $|\cL|\cdot \poly(1/\eps,\; |\cYh| )$ oracle calls to $\cA$.
\end{result}
The guarantees of this algorithm represent a significant point of departure from previous work on performative prediction.
Specifically, previous algorithms for learning performatively optimal models (for a single loss) hinged on the condition that predictions had very mild, and highly-structured (e.g. linear) impact on the induced data distributions  as in \cite{miller2021outside,jagadeesan22a}.
Conversely, within the outcome performativity restriction, we make no assumptions on the way predictions influence outcomes.
Further, the omnipredictor output in the guarantee of Theorem~\ref{result:learning} has complexity scaling as stated in Theorem~\ref{result:omni}.
In other words, our learning algorithm makes no ``realizability'' assumptions and outputs an efficient predictor, regardless of the complexity of Nature.

 
\paragraph{Universally-Adaptable Omnipredictors.}
Outcome performativity focuses attention on performative shifts in the outcome distribution as a function of the chosen decision $\yhat \in \Ych$.
In particular, it excludes performative effects in the distribution over individuals $\Xcal$.
Despite this limitation, our final result shows that we can learn performative omnipredictors that are robust to \emph{exogenous} (non-performative) shifts in the distribution over individuals.

Adapting the notion of universal adaptability, introduced in the context of statistical estimation by \cite{ua}, we show how to learn \emph{univerally-adaptable} performative omnipredictors.
Whereas performative omnipredictors guarantee optimality on a fixed marginal distribution $\D$ over individuals, 
universally-adaptable omnipredictors give the same optimality guarantee, simultaneously, over a rich class of input distribution shifts $\D_\cW$.
Each distribution in $\cD_\omega\in \cD_{\cW}$ corresponds to the reweighting of probabilities in $\cD$ by some importance weight function $\omega$ in some pre-specified class $\cW$.
\begin{result}[Informal]
Let $\cL$ be a set of bounded loss functions,  $\cH$ a hypothesis class of decision rules, and let $\cW \subseteq \{\omega: \cX \rightarrow [0,\omax]\}$. 
If the functions in $\cL$, $\cH$, and $\cW$ can be computed by circuits of size $s$, then there exists a predictor $\pt$, computable by a circuit of size at most $\poly(s, |\cYh |) \cdot \omax^2/ \eps^2$, that is a $(\cL,\Hcal,\eps)$-performative omnipredictor for every distribution over individuals $\cD_\omega \in \cD_\cW$.
\end{result}
The result follows by augmenting the class of loss functions $\cL_\cW$ to account for shifts under $\cW$, and again, applying the constructive learning algorithm from Theorem~\ref{result:learning}.
We emphasize that the learner only needs to account for the class of shifts at \emph{training} time.
At evaluation time, the decision-maker need not know anything about the underlying distribution over individuals.
Indeed, the decision-maker can simply use $\pt$ as before, post-processing to decisions $\ft_\ell(x)$ for any $\ell \in \cL$ on an input-by-input basis.


%
 


\subsection{Our Techniques:  Performative Outcome Indistinguishability}
We begin our technical overview with a simple motivating example.
Consider the following performative prediction problem.

\begin{example*}
Let individuals and decisions be encoded as signed booleans $\Xcal = \set{\pm 1}$ and $\Ych = \set{\pm 1}$, assuming $\D$ is uniform over $\Xcal$.
Suppose that Nature's outcome distribution over $\cY = \set{0,1}$ is governed by the conditional probability function
\begin{gather*}
    \ps(x,\yhat) = 1/2 + \beta x \yhat 
\end{gather*}
for any $0 < \beta < 1/2$.
Consider the goal of learning an $(\cL,\Hcal)$-performative omnipredictor for the following collection of losses and hypotheses:
\begin{itemize}
    \item $\cL = \set{\ell_0,\ell_1}$ contains two opposing steering losses, which steer outcomes towards $0$ and $1$, respectively.
    \begin{align*}
        \ell_1(x,\yhat,\ys) = 1-\ys&&\ell_0(x,\yhat,\ys) = \ys
    \end{align*}
    \item $\Hcal = \set{h_+,h_-}$ contains two decision rules over $\Xcal$ that either returns $x$ or its negation.
    \begin{align*}
        h_+(x) = x&& h_-(x) = -x
    \end{align*}
\end{itemize}
\end{example*}

We begin by considering some naive attempts to achieve the goal of performative omniprediction.
Note that $\Hcal$ actually contains an optimal decision rule for each loss in $\cL$.
In particular, for losses that steer to $0$ versus $1$, the optimal decisions minimize (or maximize) the probability that $\ys = 1$.
The decision rule $h_+$ maximizes the probability $\ps(x,h_+(x)) = 1/2 + \beta$ for all $x$, whereas $h_-$ minimizes the probability $\ps(x,h_-(x)) = 1/2-\beta$.
To obtain the omniprediction guarantee, then, we must learn a probability function that encodes the best decision under $\ell_0$ and $\ell_1$.

As such, a natural approach would be to fit a function $p:\Xcal \times \Ych \to [0,1]$ that approximates the underlying probability $\ps$, which we can then post-process for a loss $\ell$ as in \equationref{eq:opt_dr} .
To fit $p$, we can simply do supervised learning directly over triples $(x,\yhat,\ys)$, where $\yhat$ is chosen uniformly at random.
We argue that without specification (realizability) assumptions, this approach also fails.
Consider, for instance, fitting $p$ using logistic regression
\begin{align*}
	p(x,\yhat) = \frac{1}{1 + \exp(-(ax + b\yhat + c))},
\end{align*}
where $a,b,c$ are parameters of the model.
In our example, when we select $\yhat$ at random, the outcome $\ys$ is uncorrelated with each of $x$ and $\yhat$ on their own.
Thus, the optimal setting of these parameters is $a = b = c = 0$.\footnote{More concretely, since $\E[x\yhat] = \E[x\ys] = 0$, one can check that $a=b=c=0$ solves the first-order optimality conditions for the logistic regression objective $\E_{x,\yhat, \ys} -[y \log \sigma(ax +b\yhat +c) + (1-y)\log(1- \sigma(ax +b\yhat +c))]$ for $\sigma(z) = 1 / (1 + \exp(-z))$.}
Consequently, the logistic model is a constant: $p(x, \yhat) = 1/2$ for all $x \in \Xcal$ and $\yhat \in \Ych$.
Such constant predictions are completely uninformative. Clearly, they cannot suggest the optimal decision rule for any loss, let alone every loss in our collection.
The negative result here, follows because the model class for $p$ was misspecified to fit $\ps$.
In this example, of course, we simply need to run regression with quadratic terms to be well-specified.
However, without any assumptions about the complexity of $\ps$, we cannot rely on approaches that require specifying Nature's model exactly, which might have unbounded complexity.

\paragraph{Performative Outcome Indistinguishability.}

Recently, \cite{oi} introduced the notion of Outcome Indistinguishability (OI) as a new solution concept for supervised learning.
In contrast to the traditional framing of learning through loss minimization, OI defines the goal of learning through the lens of indistinguishability.
In this view, a predictive model should provide outcomes that cannot be distinguished from true outcomes from Nature.
In the world of supervised learning, OI and the closely-related notion of multicalibration \cite{hkrr} have seen broad application, including in deriving supervised omnipredictors \cite{omni,gopalan2022loss}.

Towards our goal of performative omniprediction, we adapt the paradigm of learning via outcome indistinguishability to the outcome performative setting.
In particular, we leverage the variant of outcome indistinguishability explored in a recent work of \cite{gopalan2022loss}.
In the supervised learning setting, \cite{gopalan2022loss} builds a set of indistinguishability conditions called Loss OI, which they show implies omniprediction.
We show that these conditions and the argument are even more general than originally conceived, and can be applied to the outcome performative setting.

Intuitively, we say a predictor $\pt:\Xcal \times \Ych \to [0,1]$ is \emph{performative outcome indistinguishable} if outcomes drawn according to the model $\yt \sim \pt(x,h(x))$ are indistinguishable from Nature's outcomes $\ys \sim \ps(x,h(x))$ \emph{under the distribution induced by a decision rule} $h$.
To make this notion precise, we need to specify what we mean by indistinguishability and pin down the decision rules we care to reason about.

To encode omniprediction through performative OI, our goal will be to devise a set of tests of a predictor $\pt$ that---if passed---guarantee for every loss $\ell \in \cL$, the decision rule $\ft_\ell$ is as good as any $h \in \Hcal$.
Formally, we start by building a class of tests from a collection of losses $\cL$ and a hypothesis class $\Hcal$.\footnote{Throughout, we use the notational shorthand $A \approx_\eps B$ to denote that $A \in [B-\eps,B+\eps]$.}
\begin{definition*}[Performative OI]
For an input distribution $\D$, collection of losses $\cL$, hypothesis class $\Hcal$, and $\epsilon \ge 0$, a predictor $\pt:\Xcal \times \Ych \to [0,1]$ is $(\cL,\Hcal,\epsilon)$-performative outcome indistinguishable (POI) over $\D$ if for all $\ell \in \cL$ and all $h \in \Hcal$,
\begin{gather*}
   \E_{\substack{x \sim \D\\\ys \sim \ps(x,h(x))}}[\ell(x,h(x),\ys)] \approx_\eps  \E_{\substack{x \sim \D\\\yt \sim \pt(x,h(x))}}[\ell(x,h(x),\yt)].
\end{gather*}
\end{definition*}
We emphasize how this performative OI condition is a natural notion of \emph{outcome} indistinguishability.
In particular, we note \emph{where} the predictor $\pt$ occurs in the POI conditions:  it is only used to sample outcomes according to $\pt$ in the modeled world.
Importantly, the decisions $h(x)$ are used in the sampling of $\yt \sim \pt(x,h(x))$.
Using $h(x)$ as the decision associated with $x$ to sample each outcome $y$ (under Nature and the model) ensures that $\pt$ 
encodes a reliable estimate of the loss $\ell$ if the decision rule $h$ is deployed.
Performative OI ensures that $\pt$ ``knows'' these values for each loss in the collection $\ell \in \cL$ and each hypothesis in our class $h \in \Hcal$.

While this OI condition ensures that $\pt$ captures the behavior of the hypotheses in $\Hcal$, it says nothing about the losses under its own decision rules $\ft_\ell$.
Reasoning about these decision rules is essential for performative omniprediction.
As such, we introduce an additional OI condition, which we call performative \emph{decision} OI, to ensure OI under the decision rules suggested by $\pt$.
\begin{definition*}[Performative Decision OI]
For an input distribution $\D$, collection of loss functions $\cL$, and $\epsilon \ge 0$, a predictor $\pt:\Xcal \times \Ych \to [0,1]$ is $(\cL,\epsilon)$-performative decision outcome indistinguishable (DOI) over $\D$ if for all $\ell \in \cL$,
\begin{gather*}
    \E_{\substack{x \sim \D\\\ys \sim \ps(x,\ft_\ell(x))}}[\ell(x,\ft_\ell(x),\ys)] \approx_\eps \E_{\substack{x \sim \D\\\yt \sim \pt(x,\ft_\ell(x))}}[\ell(x,\ft_\ell(x),\yt)].
\end{gather*}
\end{definition*}
Here, $\pt$ is still used to sample outcomes $\yt$, but is also used to determine the decision rule $\ft_\ell$, for each $\ell \in \cL$.
Syntactically, changing from $h \in \Hcal$ to $\ft_\ell$ is a small change, but it has significant impacts on the nature of the performative DOI condition---both in terms of its costs and the strength of its guarantees.
Critically, for a given loss $\ell$, the decision rule $\ft_\ell$ is (by definition) optimal for outcomes sampled from $\yt \sim \pt(x,\ft_\ell(x))$.
As such, indistinguishability is a powerful tool here: if the losses are indistinguishable on  modeled outcomes (where $\ft_\ell$ is optimal) and on Nature's outcomes, then $\ft_\ell$ should be optimal for Nature.

We formalize this intuition, demonstrating that, as in the supervised setting, performative OI and decision OI suffice to establish omniprediction.
Consider the loss $\ell \in \cL$ obtained by any $h \in \Hcal$ on true outcomes.
We show that the loss of $\ft_\ell$ is upper bounded by that of $h$.
\begin{proposition*}[Informal]
If a predictor $\pt$ is $(\cL,\Hcal)$-POI and $\cL$-DOI, then $\pt$ is an $(\cL,\Hcal)$-performative omnipredictor.
\end{proposition*}

\noindent\emph{Proof sketch.}~ Once the appropriate OI conditions are written down, deriving performative omniprediction is almost immediate.
\begin{align*}
    \E_{\substack{x \sim \D\\\ys \sim \ps(x,\ft_\ell(x))}}[\ell(x,\ft_\ell(x),\ys)] &\approx \E_{\substack{x \sim \D\\\yt \sim \pt(x,\ft_\ell(x))}}[\ell(x,\ft_\ell(x),\yt)]\\
    &\le \E_{\substack{x \sim \D\\\yt \sim \pt(x,h(x))}}[\ell(x,h(x),\yt)]
    \approx \E_{\substack{x \sim \D\\\ys \sim \ps(x,h(x))}}[\ell(x,h(x),\ys)]
\end{align*}
The first equality follows by $\cL$-DOI, the second equality follows by $(\cL,\cH)$-POI, and the middle inequality follows by the fact that $\ft_\ell$ is optimal over modeled outcomes. \hfill \qedsymbol
\vspace{8pt}

In other words, we have managed to reduce the task of learning performative omnipredictors to learning models satisfying performative OI conditions.
Clearly, Nature's model $\ps$ is ``indistinguishable'' from Nature (similarly, it is clear that $\ps$ is a performative omnipredictor), but the question remains whether there exist \emph{efficient} predictors $\pt$ that satisfy the performative OI conditions.
Thus, we turn our attention to learning OI predictors under outcome performativity.

\paragraph{Learning Outcome Performative Predictors.}
Despite essential differences in the notions of OI in the supervised and performative settings, we show that many of the algorithmic techniques that have become standard in the literature on multicalibration and OI can be adapted to work in the outcome performative setting.
In particular, we demonstrate that, quite generically, learning performative OI models reduces to \emph{auditing} for distinguishability.
Concretely, if there exists a loss $\ell \in \cL$ and hypothesis $h \in \Hcal \cup \set{\ft_\ell}$, such that the performative OI conditions are violated for a predictor $\pt$, we can use these ``distinguishers'' to update the model to address the violation.
This observation immediately suggests using a boosting algorithm, in the vein of \cite{hkrr}, to learn performative OI predictors.
Provided that updating based on an $\eps$-violation makes significant ``progress'' towards satisfying performative OI, then the number of auditing steps $T \le O(1/\eps^2)$ will be bounded.

While this learning paradigm of ``audit, then update'' is intuitive, there are nontrivial challenges in maintaining the efficiency of the learned predictors.
Consider, for instance, the performative decision OI constraint.
As highlighted above, the DOI constraints require that we reason about the optimal decision rule according to $\pt$.
In particular, to update based on a violation of DOI on loss $\ell$, we need to incorporate a copy of the function $\ft_\ell(\cdot)$.
This decision rule, however, is a function of $\pt(\cdot, \yhat)$ for every $\yhat \in \Ych$ (as it requires computing the argmin over $\Ych$).
Naively, then, it would seem that in every iteration where we update the model based on some $\ft_\ell$, we need to make $\card{\Ych}$ recursive oracle calls to the existing model.
Without careful consideration, the \cite{hkrr}-style learning algorithm will build a performative OI predictor $\pt$ whose complexity $s$ scales \emph{exponentially} in the number of iterations, $s \ge \card{\Ych}^T$.

To avoid this blow-up, we need to choose a more effective representation of the probability function $\pt(\cdot, \cdot)$.
We observe that, in general, the updates required for the algorithm are \emph{sparse} in $\Ych$.
As a result of this sparsity, we can save on overall computation by implementing $\pt:\Xcal \times \Ych \to [0,1]$ as a map from individuals to vectors of probabilities $\tilde{q}:\Xcal \to [0,1]^\Ych$.
Mathematically, there is a bijection between such functions; computationally, however, the representations behave very differently.
By increasing the amount of work per update by a factor of $\card{\Ych}$, we avoid making $\card{\Ych}$ recursive calls.
This strategy is reminiscent of an approach \cite{dwork2022beyond} used to learn (supervised) OI predictors for outcomes living in a large domain.
With this representation in place, an appropriate analysis reveals that the resulting predictors can be implemented in complexity that scales only \emph{polynomially} in the number of iterations and decisions.


After addressing representation issues, we study sufficient conditions to implement the auditing task efficiently, from a polynomially bounded number of samples.
Achieving performative optimality, in general, requires exploration of the consequence of using different decision rules $h \in \Hcal$, as these decision rules effect the distribution on outcomes.
Nevertheless, we show that it suffices for this exploration to be done ``offline'' via randomized assignment of decisions $\yhat \in \Ych$.
In particular, if we collect triples $\set{(x,\yhat,\ys)}$ through a randomized control trial, assigning $\yhat$ uniformly at random for each $x \in \Xcal$, and observing $\ys \sim \ps(x,\yhat)$, we avoid the need to deploy each $h \in \Hcal$.

Given access to such RCT data, we give a reduction from the task of auditing for performative OI to the task of \emph{supervised learning} for the hypothesis class $\Hcal$.
This reduction from auditing to learning is familiar in the OI framework \cite{hkrr,oi}, but critically, we go from a performative prediction task to a non-performative task.
In all, our reductions show that if we can learn the best decision rule from $\Hcal$ in a supervised learning setting, then we can learn performative omnipredictors with respect to $\Hcal$, assuming access to appropriately sampled data.

\paragraph{Universal Adaptability under Outcome Performativity.}
The OI viewpoint enables a similarly straightforward analysis of distributional robustness, via universal adaptability.
Universal adaptability is a notion introduced by \cite{ua} in the context of statistical estimation.
In the original context, a predictor is universally adaptable if it provides an efficient way to estimate statistics across many underlying input distributions.

We translate the notion of universal adaptability to the outcome performative prediction setting.
In our context, we parameterize universal adaptability by a class of importance weight functions $\cW \subseteq \set{\Xcal \to \R_{\ge 0}}$.
For a base input distribution $\D$, we define a corresponding collection of shifted distributions $\D_{\cW}$ to be the set of distributions reachable after reweighting by some $\omega \in \cW$.
\begin{gather*}
    \D_{\cW} = \set{\D_\omega : \omega \in \cW, \supp(\D_\omega) \subseteq \supp(\cD)}  \text{ and } \forall x \in \supp(\D_\omega):~~ \D_\omega(x) = \omega(x) \cdot \D(x) 
\end{gather*}
The key observation is that for any hypothesis $h$, loss function $\ell$, and importance weight function $\omega$, the expected loss over $\D_\omega$ is equal to an expected loss over $\D$, for a loss defined in terms of $\ell$ and $\omega$. \begin{align*}
	\E_{\substack{ x\sim \cD_{\omega} \\ \yt \sim \pt(x, h(x))}} [\ell(x, h(x), \yt)] = \E_{\substack{ x\sim \cD \\ \yt \sim \pt(x, h(x))}} [\ell(x, h(x), \yt) \cdot \omega(x)]
\end{align*} 
Importantly, this equality relies on the fact that 
the outcome probability functions $\ps$ and $\pt$ are defined conditional on $x$ and $\yhat$, and thus are invariant across shifts in the input distribution.

Using the result that indistinguishability implies omniprediction, by simply enforcing that $\pt$ satisfy the POI and DOI conditions relative to $\ps$ under this enriched class of loss functions $\ell \cdot \omega$, we can neatly ensure that $\pt$ is again POI and DOI, not just over $\cD$, but over every marginal distribution $\cD_\omega$. Consequently, $\pt$ must a be an omnipredictor under all of these $\cD_\omega$. The simplicity of this analysis attests to the versatility of the OI perspective and the value it provides in domains beyond supervised learning. 

\subsection{Related Work and Discussion}

Our work lies at the intersection of several areas including performative prediction, the outcome indistinguishability  and multicalibration literature, as well as other fields studying algorithmic decision-making such as contextual bandits. We briefly discuss how our results relate to previous work within these areas and conclude with some speculation regarding broader implications of our conclusions.

\paragraph{Performative Prediction.} The performative prediction framework was introduced by \cite{performative} who defined the main solution concepts and analyzed the convergence of repeated risk minimization to \emph{performatively stable} points. A decision rule $h_{\mathrm{ps}}$ is performatively stable if it is a fixed point of risk minimization,
\begin{align*}
	h_{\mathrm{ps}} \in \argmin_{h \in \cH}\E_{(x,\ys)\sim \cD(h_{\mathrm{ps}})} [\ell(x, h(x), \ys)].
\end{align*}
Subsequent work by \cite{mendler2020stochastic,drusvyatskiy2022stochastic,brown2022performative,cutler2021stochastic}
studied stochastic optimization algorithms for finding stable points in a variety of settings. 
However, these stable solutions need not be performatively \emph{optimal} as per the definition outlined in \equationref{eq:po}. 
In fact, \cite{miller2021outside} proved that performatively stable models can achieve \emph{arbitrarily worse} loss than performative optimal points.
This observation motivated the design of algorithms for findings performatively optimal decision rules for a fixed loss $\ell$\cite{miller2021outside,izzo2021learn,izzo2022learn,narang2022multiplayer,jagadeesan22a}. 
These algorithms work in the general performative prediction setup where
the model $h$ can affect the \emph{joint} distribution over pairs $(x,y)$, but make make very strong specification assumptions on how each $h$ influences the distribution, and restrict to loss functions satisfying smoothness and strong convexity.

In contrast, our results rely only on the outcome performativity assumption and mild boundedness assumptions.
The recent  work of \cite{mendler2022predicting} has also considered the outcome performativity setting, aiming to understand when performative effects are identifiable from observational data.
Short of identifiability, it remains an interesting direction for future research to give learning algorithms for performative omnipredictors from observational data.

Beyond these optimization results, previous work in performative prediction has acknowledged
the tension in performative prediction between accurate forecasting and steering.
Concretely, \cite{miller2021outside} discuss how the choice of loss function in performative prediction should balance predictive accuracy with any externalities that arise from the impacts of prediction on the observed distribution.
In a different direction, \cite{hardt2022performative} uses performativity as a lens with which to study notions of market power in economics.
As part of their analysis, they provide a decomposition of the performative risk of a classifier into terms that represent forecasting and steering.
While we consider how the choice of loss function determines the high-level objective, \cite{hardt2022performative} considers how, even for a fixed loss function, the performative risk can be decomposed into terms associated with forecasting and steering.

\paragraph{Reinforcement Learning and Contextual Bandits.} As discussed in \cite{performative}, performative prediction, and in particular outcome performativity, can be cast as reinforcement learning (RL) or contextual bandits problems. 
Individuals $x$ correspond to the contexts, decisions $\yhat$ correspond to actions, and the loss $\ell(x,\yhat, \ys)$ is captured by the reward $r(x,\yhat)$.
Due to the breadth of their definitions, most ML problems can be written as RL problems. 

Still, important issues that arise in outcome performativity---like the tension between forecasting and steering and the desire for omnipredictors---are best seen by focusing on the specific interactions between predictions $\yhat$ and outcomes $y$.
The variety of losses that can exist for a given outcome are obscured by encapsulating all feedback within an abstract reward function $r(x,\yhat)$.
Moreover, on a technical level, performativity has a richer feedback structure that can be used to design more efficient algorithms as illustrated by \cite{jagadeesan22a}.

\paragraph{Multicalibration and  Outcome Indistinguishability.}
Originally developed by \cite{hkrr} as a notion of fairness in prediction, multicalibration has seen considerable interest and application in the broader context of supervised learning.
At a high-level, multicalibration requires predictions to be \emph{calibrated}, not just overall, but even when restricting our attention to structured subpopulations.
The goal of multicalibration and other related notions of ``multi-group'' fairness \cite{kearns2018preventing,kim2018fairness,kim2019multiaccuracy,jung2021moment} is to ensure that learning occurs within important subpopulations that might otherwise be ignored.

Intuitively, the requirements of multicalibration represent a kind of indistinguishability:  calibration requires that the predicted probabilities ``look like'' real probabilities.
\cite{oi} formalizes this intuition, introducing the notion of Outcome Indistinguishability, which generalizes multicalibration.
They show tight computational equivalences between multi-group fairness notions and variants of OI.
Subsequently, OI and multicalibration have been applied in diverse contexts beyond fairness, such as distributional robustness through universal adaptability \cite{ua} and omniprediction \cite{omni}.

\paragraph{Omniprediction.}
Our work draws directly from the work on omnipredictors \cite{omni,gopalan2022loss}.
Even in the supervised learning setting, the existence of efficient omnipredictors is not at all obvious.
The main result of \cite{omni} demonstrates the feasibility of omnipredictors over any hypothesis class $\Hcal$, for the class $\cL_\mathrm{cvx}$ of all convex and Lipschitz loss functions.
This sweeping result follows by showing that a $\Hcal$-multicalibrated predictor is a $(\cL_\mathrm{cvx},\Hcal)$-omnipredictor.

A key follow-up work of \cite{gopalan2022loss} shows that omniprediction is even more general than initially thought.
This work initiates the study of omniprediction through the lens of outcome indistinguishability.
By the equivalence of OI with multicalibration, it has been clear since the work of \cite{omni} that OI captures loss minimization and omniprediction in the context of supervised learning, albeit indirectly.
\cite{gopalan2022loss} revisits the question of omnipredictors, directly through the lens of OI, studying a refined notion, which they call Loss OI.
They derive a general recipe for omnipredictors for any class of losses $\cL$, from calibration and multiaccuracy over a class derived from $\cL$ and $\Hcal$.

On a technical level, our analysis follows the OI-based approach of \cite{gopalan2022loss}.
Indeed, our proof that Performative OI and Performative Decision OI imply Performative Omniprediction follows the strategy laid out to obtain supervised omnipredictors from Loss OI.
Despite the fact that outcome performativity requires us to reason about distributional shifts induced by the predictions, the same indistinguishability framework proves effective for learning omnipredictors in our setting.

While syntactically similar to prior formulations of OI, our notion of performative OI is the first to consider outcome indistinguishability for non-supervised learning distributions.
Our objective, in this work, was to derive a notion of performative OI sufficient to imply performative omniprediction.
No doubt, further generalizations of the original OI hierarchy \cite{oi} to the (outcome) performative setting---and beyond---may prove useful.

\subsection*{Organization}
The remainder of the manuscript is organized as follows.
In Section~\ref{sec:omni}, we define performative omniprediction and performative outcome indistinguishability.
Here, we show how appropriate performative OI conditions suffice to obtain omniprediction.
In Section~\ref{sec:ua}, we establish the universal adaptability properties of performative omnipredictors.
In Section~\ref{sec:learning}, we give a generic learning algorithm for performative omnipredictors, demonstrating concrete instantiations of the algorithm using randomized control trial data.
Finally, in Section~\ref{sec:calibration}, we discuss notions of multicalibration in the context of performative prediction.
We speculate that some notions translate to the performative setting naturally, yielding efficient approaches to performative OI, while other notions seem to resist efficient translation.

\subsection*{Notation Overview}
We denote individual's features by $x \in \cX$ and
the available decisions by $\yhat \in \cYh$.
We assume that $\cX$ and $\cYh$ are discrete sets, and assume $\card{\cYh}$ is finite.
Throughout our presentation, we assume that outcomes $\cY = \set{0,1}$ are binary. The true conditional expectation over Nature's outcomes is given by $$\ps: \cX \times \cYh \rightarrow [0,1],$$
where for every individual $x \in \Xcal$ and decision $\yhat \in \Ych$, $\ps(x,\yhat)$ gives the true probability of positive outcome.
$$\ps(x,\yhat) = \Pr[\ys=1\mid x, \yhat]$$
In analogy to $\ps$, we denote the learner's predictor (i.e., the ``model of Nature'') by
$$\pt:\cX \times \cYh \rightarrow [0,1].$$
We use the shorthand $\ys \sim \ps(x,\yhat)$ to denote true outcomes drawn from the Bernoulli distribution with parameter $\ps(x,\yhat)$,
and $\yt \sim \pt(x,\yhat)$ to represent modeled outcomes drawn from the Bernoulli distribution with parameter $\pt(x,\yhat)$.
Throughout, we use $\cD$ to denote a marginal distribution over individual features $x \in \cX$.
To denote the approximate equality of expectations, we use the notational shorthand $A \approx_\eps B$ to denote that $A \in [B-\eps,B+\eps]$.

We take a loss function to be a map from individual-decision-outcome triples to a nonnegative value,
$$\ell:\cX \times \cYh \times \cY \rightarrow \R_{\geq 0}.$$
For a collection of losses $\cL$, we let $\lmax = \sup_{x,\yhat,y} \ell(x,\yhat,y)$.
For a fixed loss $\ell$, we define $\ft_{\ell}: \cX \to \cYh$ to be the optimal post-processing of $\pt$ according to $\ell$. 
More specifically, for every $x \in \cX$,
\begin{align*}
    \ft_{\ell}(x) = \argmin_{\yhat \in \cYh} \E_{\yt \sim p(x,\yhat)}[\ell(x, \yhat,\yt)].
\end{align*}
Note that $\ft_{\ell}$ is defined pointwise, without regard to the marginal distribution over $\cX$.
We use $\cH \subseteq \set{h:\Xcal \to \cYh}$ to represent a set of decision rules (i.e., a hypothesis class) which map individuals to decisions.
Typically, we assume that hypotheses $h \in \Hcal$ have an efficient representation. 



%% file: performative-omni.tex
\section{Performative Omniprediction}
\label{sec:omni}

Introduced in the supervised learning setting by \cite{omni},
an omnipredictor is an outcome probability model that suggests an optimal decision rule for any loss within a class of loss functions.
To discuss this concept more formally, we first review the definition optimality with respect to a fixed loss function, under our setting of outcome performativity.
\begin{definition}[Performative Optimality]
For input distribution $\D$, loss function $\ell$, hypothesis class $\Hcal$, and $\eps \ge 0$, a decision rule $f:\Xcal \to \Ych$ is $(\ell,\Hcal,\eps)$-performatively optimal over $\D$ if 
\begin{gather*}
\E_{\substack{x \sim \D\\\ys \sim \ps(x,f(x))}}[\ell(x, f(x),\ys)] \le \min_{h \in \Hcal}\E_{\substack{x \sim \D\\\ys \sim \ps(x,h(x))}}[\ell(x, h(x),\ys)] + \eps.
\end{gather*}
\end{definition}
That is, a decision rule $f$ is performative optimal if $f$ obtains expected loss under the induced outcome distribution (i.e. performative risk) competitive with the best hypothesis in some reference class $h \in \Hcal$.
\footnote{To review, as in \cite{performative} we refer to  $\E_{x \sim \D,\ys \sim \ps(x,h(x))}[\ell(x, h(x),\ys)]$ as the performative risk of $h$ on $\ell$.}
A simple, but key observation is that if we knew the model that generates outcomes for each individual perfectly, then performative optimality is straightforward to achieve.
Given an outcome model $\ps$ that specifies, for a given individual $x \in \Xcal$ and decision $\yhat \in \Ych$, the probability of outcome $y \sim \ps(x,\yhat)$, the optimal decision rule for any loss function can be determined by a simple optimization over the choice of decisions $\yhat$, as follows.
\begin{fact}[Optimal Decision Rule]
\label{fact:opt}
Fix an outcome model $\ps:\Xcal \times \Ych \to [0,1]$.
For a loss $\ell:\Xcal \times \Ych \times \cY \to \R_{\ge 0}$, there exists a globally optimal decision rule $f_{\ell}^*:\Xcal \to \Ych$ for $\ell$ under $\ps$, given by the pointwise loss-minimizer,
\begin{gather}
\equationlabel{eq:optimal_rule}
    f^*_{\ell}(x) = \argmin_{\yhat \in \Ych}\E_{\ys \sim \ps(x,\yhat)}[\ell(x,\yhat,\ys)].
\end{gather}
\end{fact}
For instance, when $\Ych$ is discrete and finite, then the optimal prediction can be computed by linearly enumerating over $\cYh$, and computing a finite sum of two terms each time.
We will often consider the optimal decision rule $\ft_{\ell}$ after post-processing an approximate outcome model, or predictor, denoted as $\pt$. 

With these definitions of performative optimality and optimal decision rule in place, we are ready to define performative omniprediction.
An omnipredictor is an outcome model $\pt$ where for \emph{every} loss $\ell \in \cL$ within some collection, the optimal decision rule derived from $\pt$ is performatively optimal with respect to some class of hypothesis $\calH$.
\begin{definition}[Performative Omniprediction]
For input distribution $\D$, collection of loss functions $\cL$, hypothesis class $\Hcal$, and $\eps \ge 0$, a predictor $\pt:\Xcal \times \Ych \to [0,1]$ is an $(\cL,\Hcal,\eps)$-performative omnipredictor over $\D$ if for all $\ell \in \cL$, the optimal decision rule
$\ft_\ell$ is $(\ell,\Hcal,\eps)$-performative optimal.
\end{definition}
Omniprediction is a very strong solution concept.
Whereas the optimal decision rule typically depends intimately on the chosen loss, an omnipredictor needs to encode the optimal decision rule for every loss in $\cL$, even if these losses encode very different preferences over predictions.
It is not hard to see that the optimal predictor is an omnipredictor for any hypothesis and loss class.
\begin{corollary}
For any input distribution $\D$, collection of loss functions $\cL$, and hypothesis class $\Hcal$, the optimal predictor $p^*:\Xcal \times \Ych \to [0,1]$ is an $(\cL,\Hcal,0)$-performative omnipredictor over $\D$.
\end{corollary}
This corollary follows directly from Fact~\ref{fact:opt}, because the optimal predictor $p^*$ gives the true probability law governing the performative outcome distribution. 
Still, the optimal predictor may be of arbitrary complexity and is generally inaccessible.
The question remains whether \emph{efficient} performative omnipredictors exist, and if so, how to learn them.
To attack this question, we introduce a generalization of the outcome indistinguishability framework to the outcome performativity setting.

\subsection{Performative Outcome Indistinguishability}
Outcome Indistinguishability (OI) was introduced by \cite{oi} as an alternative paradigm for supervised learning.
Rather than focusing on loss minimization, OI formalizes learning as a computational indistinguishability condition.
In this view, a predictor should produce outcomes that are indistinguishable from Nature's outcome distribution.
While OI can encode classic learning goals like loss minimization, the abstraction is quite generic and amenable to modern supervised learning desiderata, like fairness \cite{hkrr} and distributional robustness \cite{ua}. 

Here, we propose an indistinguishability definition for the performative world.\footnote{In its original formulation, OI is a hierarchy of related notions. We generalize the framework to our setting, focusing on notions of performative OI that will imply performative omniprediction.  Understanding a full generalization of the OI framework to the performative setting is an interesting question for future investigations.}
This definition extends what \cite{gopalan2022loss} refer to as Hypothesis OI in the supervised setting.
\begin{definition}[Performative OI]
For input distribution $\D$, collection of losses $\cL$, hypothesis class $\Hcal$, and $\epsilon \ge 0$, a predictor $\pt:\Xcal \times \Ych \to [0,1]$ is $(\cL,\Hcal,\epsilon)$-performative outcome indistinguishable (POI) over $\D$ if for all $\ell \in \cL$ and all $h \in \Hcal$,
\begin{gather*}
   \big| \E_{\substack{x \sim \D\\\ys \sim \ps(x,h(x))}}[\ell(x,h(x),\ys)] -  \E_{\substack{x \sim \D\\\yt \sim \pt(x,h(x))}}[\ell(x,h(x),\yt)] \big| \leq \epsilon
\end{gather*}
\end{definition}
In this definition, we fix our collection of distinguishers to be parameterized by a collection of loss functions and a hypothesis class.
The POI condition states that, even when the outcome distribution can depend nontrivially on the hypothesis value $h(x)$, the outcomes $\ys$ and $\yt$ are indistinguishable, as measured by the expected loss of each hypothesis.
Note that the distinguishers take as input the individual $x$, the decision $h(x)$, and either Nature's outcome $\ys$ or the modeled outcome $\yt$.
In particular, these distinguishers do not receive access to the predictions $\pt(x,h(x))$ themselves.\footnote{In the language of \cite{oi}, this notion corresponds to the ``No-Access'' level of the OI hierarchy.  In principle, we could also extend the upper levels to the performative setting as well.  We comment this issue further within our discussion of performative calibration in Section~\ref{sec:calibration}.}

As a step towards obtaining omniprediction, we require indistinguishability between $\pt$ and $\ps$ not just under the reference decision rules $h$, but also under the optimal decision rules $\ft_{\ell}$ derived from $\pt$.
This motivates the notion of Performative Decision OI, which extends the idea of decision calibration, introduced in \cite{zhao2021calibrating}, and decision OI, introduced in \cite{gopalan2022loss}, to the performative setting.
\begin{definition}[Performative Decision OI]
For input distribution $\D$, collection of loss functions $\cL$, and $\epsilon \ge 0$, a predictor $\pt:\Xcal \times \Ych \to [0,1]$ is $(\cL,\epsilon)$-performative decision outcome indistinguishable (DOI) over $\D$ if for all $\ell \in \cL$,
\begin{gather*}
    \big| \E_{\substack{x \sim \D\\\ys \sim \ps(x,\ft_\ell(x))}}[\ell(x,\ft_\ell(x),\ys)] - \E_{\substack{x \sim \D\\\yt \sim \pt(x,\ft_\ell(x))}}[\ell(x,\ft_\ell(x),\yt)] \big| \leq \epsilon.
\end{gather*}
\end{definition}
Operationally, DOI allows us to sample outcomes $\yt \sim \pt(x,\ft_\ell(x))$ from our model of Nature, evaluate the expected loss of $\ell(x,\ft_\ell(x),\yt)$, and be confident that it is close to the loss on outcomes sampled from Nature $\ys \sim \ps(x,\ft_\ell(x))$.

Note that, technically, the indistinguishability conditions in Performative OI and Performative Decision OI look the same, but just refer to different hypothesis classes; that is, $(\cL, \eps)$-Performative Decision OI can be phrased as $(\cL, \set{\ft_\ell : \ell \in \cL}, \eps)$-Performative OI.
We make a distinction between these notions because, semantically, the hypothesis class $\set{\ft_\ell : \ell \in \cL}$ is derived from the predictor $\pt$, whereas $h \in \Hcal$ is independent of $\pt$.
As we will see later, this semantic difference manifests as a concrete difference in the computational complexity of achieving each notion of indistinguishability.


\subsection{Performative Omniprediction via OI}
With these definitions in place, we can prove our first main result: performative omniprediction from performative outcome indistinguishability.
One of the main benefits of studying the problem from the indistinguishability lens is that it enables an especially clean and simple analysis.
The proof strategy we employ here follows the proof of omniprediction in the supervised learning world by \cite{gopalan2022loss}.
Curiously, the proof only needs one direction of the indistinguishability inequalities.
\begin{theorem}
\theoremlabel{thm:omni}
Fix an input distribution $\D$, collection of losses $\cL$, hypothesis class $\Hcal$, and $\epsilon \ge 0$.
Suppose that $\pt:\Xcal \times \Ych \to [0,1]$ is $(\cL,\epsilon)$-performative decision OI and $(\cL,\Hcal,\epsilon)$-performative OI.
Then, $\pt$ is a $(\cL,\Hcal,2\epsilon)$-performative omnipredictor.
\end{theorem}

\begin{proof}
The proof exploits the fact that for each loss $\ell \in \cL$, $\ft_\ell$ is the optimal decision rule for $\ell$ under $\pt$.
Fix a loss $\ell \in \cL$.
First, we upper bound the loss achieved by $\ft_\ell$ on real outcomes $\ys$ in terms of the loss on modeled outcomes $\yt$. Under  $(\cL, \epsilon)$-performative decision OI,
\begin{align*}
\E_{\substack{x \sim \D \\ \ys \sim \ps(x,\ft_{\ell}(x))}}[\ell(x, \ft_{\ell}(x),\ys)]
&\le \E_{\substack{x \sim \D\\\yt \sim \pt(x,\ft_{\ell}(x))}} [\ell(x, \ft_\ell(x),\yt)] + \epsilon. \label{ineq:calibration}
\end{align*}
Next, we relate the expected loss achieved by $\ft_\ell$ on modeled outcomes $\yt \sim \pt(x,\ft_\ell(x))$ versus that of other decision rules $h$.
By its definition, $\ft_\ell(x)$ is the optimal decision over any $\yhat \in \Ych$ for the loss $\ell(x,\yhat,\yt)$ under $\yt \sim \pt(x,\yhat)$.
So, averaging over the distribution on inputs $x \sim \D$, the loss of $\ft_\ell$ is upper bounded by the loss of any other decision rule $h$, and in particular those in $\Hcal$:
\begin{align*}
\E_{\substack{x \sim \D\\\yt \sim \pt(x,\ft_{\ell}(x))}} [\ell(x, \ft_\ell(x),\yt)] &\le \E_{\substack{x \sim \D\\\yt \sim \pt(x,h(x))}}[\ell(x, h(x),\yt)]. 
\end{align*}
Finally, by $(\cL,\Hcal,\eps)$-POI, we upper bound the loss achieved by $h$ on real outcomes $\ys$ by that achieved on modeled outcomes $\yt$.
\begin{align*}
\E_{\substack{x \sim \D\\\yt \sim \pt(x,h(x))}}[\ell(x, h(x),\yt)] &\le \E_{\substack{x \sim \D\\\ys \sim \ps(x,h(x))}}[\ell(x, h(x),\ys)] + \epsilon.
\end{align*}
Combining these three inequalities, 
\begin{gather*}
    \E_{\substack{x \sim \D \\ \ys \sim \ps(x,\ft_{\ell}(x))}}[\ell(x, \ft_{\ell}(x),\ys)] \le \E_{\substack{x \sim \D\\\ys \sim \ps(x,h(x))}}[\ell(x, h(x),\ys)] + 2\epsilon,
\end{gather*}
so $\pt$ is a $(\cL,\Hcal, 2\eps)$-performative omnipredictor.
\end{proof}

%% file: universality.tex
\section{Universal Adaptability}
\label{sec:ua}

In addition to minimizing expected risk, performative omniprediction can also be viewed as a guarantee of robustness.
So far, we've seen how a performative omnipredictor induces optimal predictions $\yhat$ even if these predictions lead to endogenous shifts in the distribution over outcomes $\ys$. 
In this section, we argue that with little additional work, the OI framework can be adapted to yield performative omnipredictors that are robust to exogenous shifts in the marginal distribution over individuals $x$.\footnote{By endogenous we mean that the distribution shift is caused by the act of prediction itself, which is considered in the outcome performativity framework. Exogeneous shifts are not influence by predictions. They refer to changes in the data distribution caused by factors like a change in external environment, or the passage of time.}

The results here build on the recent work of \cite{ua}, who introduced a notion of \emph{universal adaptability} in the context of statistical inference problems.
In our context, universal adaptability may be interpreted as a guarantee that the performative omnipredictor properties hold, not only on the original input distribution $\D$, but also on a broad family of shifts of this input distribution $\cD$.
In particular, we show that by augmenting the class of loss functions, we can learn an outcome prediction model $\pt$ that can handle exogenous shifts in the input distribution, while still maintaining performative optimality.

We parameterize universal adaptability  by a class of importance weight functions $\cW \subseteq \set{\Xcal \to \R_{\ge 0}}$.
For a base input distribution $\D$, we define a corresponding collection of shifted distributions $\D_{\cW}$ to be the set of distributions reachable after reweighting the probabilities in $\cD$ by some $\omega \in \cW$.
\begin{gather*}
    \D_{\cW} = \set{\D_\omega : \omega \in \cW, \; \supp(\D_\omega) \subseteq \supp(\cD) }, 
    \textrm{ where } \forall x \in \supp(\D_\omega),~~ \D_\omega(x) = \omega(x) \cdot \D(x) 
\end{gather*}
Note that to yield a valid probability distribution $\D_\omega$ it is necessary and sufficient that the importance weight function $\omega$ have unit weight over $\D$; that is, for any $\omega \in \cW$, $\E_{x \sim \D}[\omega(x)] = 1$.\footnote{Other properties of $\cW$ will affect whether universal adaptability is feasible, but not its definition.
We discuss these issues further in Section~\ref{sec:learning}.}
Given an importance weight class $\cW$, we say that an omnipredictor is universally adaptable if it is an omnipredictor over any $\D_\omega \in \D_{\cW}$.
\begin{definition}[Universal Adaptability]
For input distribution $\D$, weight class $\cW$, collection of losses $\cL$, hypothesis class $\Hcal$, and $\eps \ge 0$, a performative omnipredictor is $\cW$-universally adaptable over $\D$ if $\pt$ is an $(\cL, \Hcal, \eps)$-performative omnipredictor over every $\D_\omega \in \D_{\cW}$.
\end{definition}
Note that universal adaptability guarantees robustness under \emph{exogeneous} shifts in the marginal distribution over $\cX$, not under \emph{endogeneous} shifts in the input distribution induced by the act of prediction.
The distributional robustness is with respect to shifts that are defined in advance, independent of the chosen decision rule.
Under the guarantees of universal adaptability, the prevalence of various individuals may vary, but the response of any specific individual $x$ to a prediction $\yhat$, as measured by the distribution $\ps$ governing the outcome $\ys$, remains the same.
Intuitively, this type of robustness is the best that we can hope for without explicitly modeling how the predictions $\yhat \in \Ych$ change the distribution over individuals $x \in \Xcal$, which $\pt$ does not model. If predictions $\yhat$ affect both $x$ and $\ys$, it is not at all obvious to us what invariant property of Nature we should choose to model. We believe these are important questions for future work.

 One consequence of this adaptability definition is that any model $\pt$ that is an omnipredictor for a class of distributions $\D_{\cW}$ must also be an omnipredictor for any mixture distribution with components drawn from this class.  We say that a distribution $\cD_{m}$ is a mixture distribution if for all $x \in \cX$, 
\begin{align*}
	\Pr_{\cD_m}[X=x] = \sum_{\omega} \lambda_\omega \Pr_{\cD_\omega}[X=x]  \text{ where } \cD_\omega \in \cD_{\cW},\; \lambda_\omega \geq 0 \text{ for all } \omega \text{ and } \sum_\omega \lambda_\omega=1.
\end{align*}
We denote by $\mathsf{mixt}(\cD_{\cW})$ the set of all such mixture distributions $\D_m$.

\begin{proposition}
\propositionlabel{prop:mixtures}
Let $\cD_\cW$ be a set of distributions over $\cX$. If $\pt$ is a $(\cL, \cH, \epsilon)$-performative omnipredictor over every $\cD_\omega \in \cD_{\cW}$, then it is also a $(\cL, \cH, \epsilon)$-performative omnipredictor over every $\cD_m$ in $\mathsf{mixt}(\cD_{\cW})$.
\end{proposition}
\begin{proof}
Fix a loss $\ell \in \cL$, a hypothesis $h \in \cH$ and a distribution $\cD_m$ in $\mathsf{mixt}(\cD_{\cW})$. Then,
\begin{align*}
\E_{\substack{x \sim \D_m \\ \ys \sim \ps(x,\ft_{\ell}(x))}}[\ell(x, \ft_{\ell}(x),\ys)] &= \sum_\omega \lambda_\omega \cdot \E_{\substack{x \sim \cD_i \\ \ys \sim \ps(x,\ft_{\ell}(x))}}[\ell(x, \ft_{\ell}(x),\ys)]  \\ 
& \leq \sum_\omega \lambda_\omega \cdot \E_{\substack{x \sim \cD_\omega \\ \ys \sim \ps(x,h(x))}}[\ell(x, h(x),\ys)] + \sum_\omega \lambda_\omega \epsilon \\
& = \E_{\substack{x \sim \D_m \\ \ys \sim \ps(x,h(x))}}[\ell(x, h(x),\ys)] + \epsilon
\end{align*}
The first line follows by expanding the definition of the mixture distribution and the second by the omniprediction guarantee on mixture components. In the last line we again applied the definition of a mixture and the fact that the $\lambda_\omega$ sum to 1. Because the inequalities hold for every $h\in \cH$, it must be the case that $\pt$ is an $(\cL, \cH, \epsilon)$-omnipredictor for every mixture distribution .
\end{proof}


We establish universal adaptability for performative omnipredictors by augmenting the loss class $\cL$ using the weight class $\cW$.
Specifically, we define the augmented loss class $\cL_\cW$ as the class of losses $\ell \in \cL$ reweighted by importance weight functions $\omega \in \cW$.
\begin{gather*}
    \cL_{\cW} = \set{\ell_\omega : \ell \in \cL,\ \omega \in \cW}\\
    \textrm{where }\forall x \in \Xcal,\yhat \in \Ych,y \in \cY:~~~ \ell_\omega(x,\yhat,y) = \omega(x) \cdot \ell(x,\yhat,y)
\end{gather*}
With this class of losses in place, we argue that universally-adaptable performative omniprediction is, again, a consequence of performative outcome indistinguishability.
\begin{proposition}
\proplabel{prop:adaptability}
For a base input distribution $\D$, weight class $\cW$, collection of losses $\cL$, hypothesis class $\Hcal$, and $\eps \ge 0$, if a predictor $\pt$ is $(\cL_{\cW},\Hcal,\eps)$-performative OI and $(\cL_\cW,\eps)$-performative decision OI over $\D$, then $\pt$ is an $(\cL,\Hcal,2 \eps)$-performative omnipredictor that is $\cW$-universally adaptable over $\D$.
\end{proposition}
\begin{proof}
The proposition follows as a corollary of \theoremref{thm:omni}.
The key observation is that multiplying by the importance weight $\omega(x)$ allows us to switch from an expectation over $\D$ to an expectation over $\D_\omega$.
By the definition of $\D_\omega$, we have that for supported $x \in \Xcal$, $\omega$ is the odds ratio,
\begin{gather*}
    \omega(x) = \frac{\D_\omega(x)}{\D(x)}.
\end{gather*}
Further, by the definition of $\ell_\omega$, for any $h:\Xcal \to \Ych$ and any outcome probability model $p$, the following equality of expectations holds
\begin{gather}
\equationlabel{eqn:reweight}
    \E_{\substack{x \sim \D\\y \sim p(x,h(x))}}[\ell_\omega(x,h(x),y)]
    = \E_{\substack{x \sim \D\\y \sim p(x,h(x))}}[\omega(x) \cdot \ell(x,h(x),y)]
    = \E_{\substack{x \sim \D_\omega\\y \sim p(x,h(x))}}[\ell(x,h(x),y)],
\end{gather}
where we rely on the identity that for any function $g:\Xcal \to \R$,
\begin{gather*}
    \E_\D[g(x) \cdot \omega(x)] = \E_\D[g(x) \cdot \D_\omega(x)/\D(x)] = \E_{\D_\omega}[g(x)].
\end{gather*}
The equality in \equationref{eqn:reweight} immediately implies that if $\pt$ is $(\cL_\cW,\Hcal,\eps)$-POI over $\D$, then $\pt$ is $(\cL,\Hcal,\eps)$-POI over every $\D_\omega \in \D_\cW$.
That is, by applying the identity to the expectation under Nature's outcomes $\ys \sim \ps(x,h(x))$ and separately to the expectation under the modeled outcomes $\yt \sim \pt(x,h(x))$, $(\cL_\cW,\Hcal,\eps)$-performative OI imples that we obtain indistinguishability for all $\ell \in \cL, h \in \Hcal$ and $\D_\omega \in \D_\cW$:
\begin{gather*}
\big|\E_{\substack{x \sim \D_\omega\\\ys \sim \ps(x,h(x))}}[\ell(x,h(x),\ys)]
-
\E_{\substack{x \sim \D_\omega\\\yt \sim \pt(x,h(x))}}[\ell(x,h(x),\yt)] \big| \leq \epsilon.
\end{gather*} 
The corresponding statement for performative decision OI is a bit more subtle.
Whereas above, the decision rules $h \in \Hcal$ do not depend in any way on $\omega$, the optimal decision rule $\ft_{\ell_\omega}$ based on $\pt$, is allowed to depend on the specified loss and, thus, on $\omega$.
Still, we argue that for any $\omega$, $\ft_{\ell_\omega} = \ft_\ell$.
This equality follows by the fact that the optimal decision rule is chosen pointwise, for each $x \in \Xcal$.
In particular, for all $x \in \Xcal$, scaling the loss by $\omega(x)$ changes the scale of the optimization, but not the minimizer:
\begin{gather*}
    \ft_{\ell_\omega}(x) = \argmin_{\yhat \in \Ych} \E_{y \sim \pt(x,\yhat)}[\omega(x) \cdot \ell(\yhat,y)] = \argmin_{\yhat \in \Ych} \E_{y \sim \pt(x,\yhat)}[\ell(\yhat,y)] = \ft_{\ell}(x).
\end{gather*}
Thus, the same identities from above can be applied to prove that if $\pt$ is $(\cL_\omega,\eps)$-DOI for a fixed distribution $\cD$, then it is also $(\cL, \epsilon)$-DOI for every $\D_\omega \in \D_\cW$.
The proposition follows by applying \theoremref{thm:omni} separately over each $\D_\omega \in \D_\cW$.
\end{proof}

Before moving on, we highlight that designing omnipredictors requires the learner to account for possible shifts in the distribution at \emph{training} time, not at test time. At test time, the learner simply chooses predictions $\yhat$ according to the function $\ft_\ell$, without needing to first infer what the underlying distribution $\cD$ may be. The decision rule $\ft_\ell$ is simultaneously optimal for all of them. 
This design choice shifts the burden of technical sophistication and expertise from the user of the system to its designer. The user is free to focus on the choice of loss function $\ell$ to balance between forecasting and steering knowing that naive usage of $\ft_\ell$ is guaranteed to work.

%% file: learning.tex
\section{Learning Algorithms for Performative Omniprediction}
\label{sec:learning}

\begin{figure}[t!]
\setlength{\fboxsep}{2mm}
\begin{boxedminipage}{\textwidth}
\begin{center}
{\centering{\underline{Performative OI Boost (POI-Boost)}}}
\end{center}

\vspace{2mm}
{\bf Input:} Set of losses $\cL$, hypotheses $\cH$, distribution $\cD$, tolerance $\eps > 0$\\
\vspace{-2mm}

{\bf Initialize:} $q^{(1)}(\cdot) \gets [1/2, \dots, 1/2] \in [0,1]^{|\cYh|}$  
\\

{\bf For} $t = 1,2, \ldots$
\begin{itemize}
	\item {\bf If:} there exists $\mathbf{a)}$  $(h,\ell) \in \cH \times \cL$ \quad or,  \quad   $\mathbf{b)}$ $(h, \ell) \in  \{(f_{\ell, t}, \ell): \ell \in \cL \}$
	\begin{align}
	\equationlabel{eq:constraint}
		| \underbrace{\E_{\substack{x \sim \cD \\ \yt \sim \pt(x, h(x))}}[\ell(x, h(x),\yt)] - \E_{\substack{x \sim \cD \\ \ys \sim \ps(x, h(x))}}[ \ell(x, h(x),\ys)]}_{\mathsf{err}_t} | \geq \eps
	\end{align}

	{\bf Then:} Update the representation $q^{(t)}$: 
	\begin{align}
	\label{eq:predictor_update_rule}
		q^{(t+1)}(\cdot) \leftarrow \Pi \left(\; q^{(t)}(\cdot) - \eta^{(t)} \; v_{\ell, h}(\cdot) \;\right)
	\end{align}
	where $\eta_t = - \epsilon\cdot \sign(\mathsf{err}_t) \; / \; \lmax$ and 
	\begin{align}
	\equationlabel{eq:gradient_def}
	v_{\ell, h}(\cdot) =	 \begin{bmatrix}
	 (\ell(\cdot, \yhat_1, 1) - \ell(\cdot, \yhat_1, 0))  \;\ind\{h(\cdot) = \yhat_1 \}\\ 
	 \dots  \\
	 (\ell(\cdot, \yhat_k, 1) - \ell(\cdot, \yhat_k, 0)) \; \ind\{h(\cdot) = \yhat_k \}
	 \end{bmatrix}\in \R^{|\cYh|},\quad  \cYh = \{\yhat_1, \dots, \yhat_k\}
	\end{align}

	\item {\bf Else:} terminate and return the function $\pt(x,\yhat) = q^{(t)}(x)[\;\yhat\;]$
\end{itemize}
\vspace*{2mm}
\end{boxedminipage}

\caption{Algorithm for generating performative omnipredictors. 
The algorithm proceeds by repeatedly verifying whether the intermediate predictors $p^{(t)}$ satisfy the POI definition, outlined in $\mathbf{a)}$, as well as the DOI definition, outlined in $\mathbf{b)}$. 
If neither is violated, the procedure terminates. 
Otherwise, the algorithm implicitly updates the representation $q^{(t)}$ of the predictor $p^{(t)}$. Given an input $x$, $q^{(t)}(x)$ is a vector of length $|\cYh|$ whose $\yhat$ entry, $q^{(t)}(x)[\;\yhat\;]$, represents $p^{(t)}(x,\yhat)$. 
The operator $\Pi$ clips entries of its input vector to lie in $[0,1]$. 
For the sake of clarity, here we present the simplest version of the algorithm where the search outlined in \equationref{eq:constraint} is proper, however this condition can be easily relaxed as discussed in \sectionref{subsec:csc}.}
\figurelabel{fig:algorithm} 
\end{figure}



In this section, we introduce a general purpose algorithm, POI-Boost, which provably returns a performative omnipredictor $\pt$ for any class of hypothesis $\cH$ and collection of losses $\cL$.
Our algorithmic approach is centered on establishing two reductions. First, we prove that, similar to previous work in the OI literature, learning Performative OI predictors reduces to the problem of \emph{auditing} for outcome indistinguishability.

The auditing problems we reduce to involve determining whether the losses under the decision rules in $\Hcal$ and $\set{\ft_\ell}$ are the same for Nature's outcomes and our modeled outcomes \emph{under outcome performativity}.
While in the supervised learning setting we only need to reason about a single outcome distribution, in our setting, 
different different decision rules induce different distributions over outcomes, and we want to audit for indistinguishability over each of these induced distributions.
Despite this challenge, we show that, given access to appropriately randomized data, we can reduce this performative auditing problem to standard supervised learning primitives.
In this second reduction, we make use of computational and statistical assumptions:  access to an appropriate supervised learner (computational) and access to randomized control data (statistical).

While we instantiate our algorithm with specific computational and statistical assumptions, the framework for learning is completely generic and modular.
In particular, any solution to the auditing problem can be used to implement the algorithm.
It stands to reason that our assumptions could be relaxed in the future, or that in certain settings, incomparable assumptions lead to more effective auditing, which in turn would lead to more efficient learning of performative omnipredictors.

\subsection{Reducing Indistinguishability to Auditing}

We start by establishing our first reduction. We prove that the POI-Boost algorithm (\figureref{fig:algorithm}) returns a performative omnipredictor $\pt$ after a small, polynomial number of calls to an auditing subroutine (described in \equationref{eq:constraint} \& \figureref{figure:auditing}), without yet describing the runtime or sample complexity of the auditing step itself. We address these questions in the next subsection. The algorithm works for any outcome performative problem where the number of predicted labels $\cYh$ is finite and the loss functions are bounded.

\paragraph{Representing Predictors.} For the sake of our analysis, it is helpful to distinguish between the predictor $\pt:\cX \times \cYh \rightarrow [0,1]$ as a function, and the implementation of $\pt$ in code. 
In our learning algorithm, we represent the function $\pt$ in terms of vector-valued functions $\qt: \cX \rightarrow [0,1]^{|\cYh|}$.
Given $x \in \cX$, $\qt(x)$ is a vector of length $|\cYh|$ whose $\yhat$ entry, $q^{(t)}(x)[\;\yhat\;]$, represents $\pt(x,\yhat)$. 

Of course, there is a correspondence between these functions $\pt$ and $\qt$ where each $\pt$ leads to a unique $\qt$ and vice versa. 
The key difference is that $\qt(x)$ returns $\pt(x,\yhat)$ for all $\yhat \in \Ych$ in a single function call, while computing the same information using the direct $\pt$ representation would require $|\cYh|$ functions calls. 
While this might seem like a minor detail, these representations have meaningful differences in terms of the circuit complexity of performative omnipredictors.
Crucially, for any loss $\ell$, to compute $\ft_\ell(x)$, we need the value of of $\pt(x,\yhat)$ for all $\yhat \in \Ych$, and this computation can be performed using a single call to $\qt$.
In other words, we perform $\card{\Ych}$ times the work per call to $\qt$ to avoid $\card{\Ych}$ \emph{recursive} calls to the predictor within the construction.
Avoiding further calls to these functions avoids branching factors and an exponential blowup in the complexity of the resulting predictors as we illustrate.


\paragraph{Algorithm Description.} As outlined in \figureref{fig:algorithm}, POI-Boost is an iterative algorithm which nonparametrically constructs a predictor $\pt$, represented in terms of a vector-valued function $\qt$, by stringing together copies of circuits which compute losses $\ell$ and decision rules $h$. 
At each iteration, the algorithm first appeals to auditing subroutines to check if there: is $a)$ a pair $h,l$ for which the current predictor $p^{(t)}$, fails the performative OI guarantee, or $b)$ a loss function $\ell$ for which the decision rule $f_{\ell,t}$ fails the decision OI guarantee. 
If neither condition is violated, then the algorithm terminates since $p^{(t)}$ satisfies both indistinguishability conditions and consequently must be an omnipredictor as per \theoremref{thm:omni}.

On the other hand, if one of these conditions is violated, we perform an update to the representation $q^{(t)}$ of the current predictor $p^{(t)}$. These updates nudge the predictor closer to $\ps$ by essentially performing gradient descent in function space \cite{mason1999boosting}. These updates are done implicitly in the sense that we can update the representation $q^{(t)}$ for all $x$ in $\cX$ simultaneously by simply adding a copy of the circuit computing $v_{\ell,h}: \cX \rightarrow \R^{|\cYh|}$ (\equationref{eq:gradient_def}) which is defined in terms of a loss $\ell$ and decision rule $h$. By bounding the total number of updates via a potential argument, we can ensure that we don't add too many copies of these functions so that the final predictor is computationally efficient. 


\begin{proposition}
\proplabel{prop:iteration_bound}
The POI-Boost algorithm described in Figure 1 terminates in at most $   |\cYh|  \lmax^2 / \eps^2$ many iterations and returns a predictor $\pt$ that is
$(\cL, \cH, \epsilon)$-performative OI and $(\cL, \epsilon)$-performative decision OI. Consequently, $\pt$ is a $(\cL, \cH, 2\epsilon)$-performative omnipredictor.
\end{proposition}
\begin{proof}

The guarantee that $\pt$ is performative decision OI and performative OI follow directly from the termination criterion. Therefore, the proposition follows from proving that this termination criteria is met within the stated number of iterations. 


The key insight is that if the indistinguishability constraint in \equationref{eq:constraint} is violated for any $\ell$ or $h$, then updating the representation $q^{(t)}$ ensures that we will have made nontrivial progress on a common potential function. 
Since this potential is bounded from above and below, and we make nontrivial progress with every update, the total number of updates must be bounded. 
In more detail, first, note that for any model $p$, loss $\ell$, hypothesis $h$, and $\cY=\{0,1\}$:
\begin{align*}
	\E_{\substack{x \sim \D\\ y \sim p(x,h(x))}} [\ell(x, h(x), y)] &= \E_{x} \E_{y\mid x} [\ell(x, h(x), y)] \\
	&= \E_{x}[\ell(x, h(x),0) + \left(\ell(x, h(x), 1) - \ell(x, h(x), 0)\right) \cdot p(x,h(x))].
\end{align*}
From this rewriting, and the definition of  $v_{\ell, h}$ in \equationref{eq:gradient_def}, the difference in performative risks between the predictors $p^{(t)}$ and $\ps$ for a pair $\ell, h$ can be expanded as,
\begin{align}
	\E_{\substack{x \sim \D\\ y_t \sim p^{(t)}(x,h(x))}} [\ell(x, h(x), y_t)] - \E_{\substack{x \sim \D\\ \ys \sim \ps(x,h(x))}} [\ell(x, h(x), \ys)]
 =\E_x \langle q^{(t)}(x)- q^*(x),\; v_{\ell,h}(x) \rangle.
	\equationlabel{eq:test_func_diff}
\end{align}

Now consider the potential, written in terms of the representations $q^{(t)}$, $\E_x \norm{q^{(t+1)}(x) - q^*(x)}^2$. By definition of the update rule in the algorithm, this potential is equal to:
\begin{align*}
 	 \E_x \norm{\Pi \left(q^{(t)}(x)   - \eta^{(t)} v_{\ell, h}(x)  \right) -
 q^*(x)}^2.
 \end{align*}
Because the projection (or clipping) operator $\Pi$ can only decrease the distance to $\ps$, if an update is performed, the difference between potentials at adjacent time steps, 
\begin{align*}
 \E_x \norm{q^{(t+1)}(x) - q^*(x)}^2 - \E_x \norm{q^{(t)}(x) - q^*(x)}^2,
\end{align*}
is upper bounded by the sum of  two terms,
\begin{align*}
 &-2 \eta_t \E_x\langle q^{(t)}(x)- q^*(x), v_{\ell, h}(x)\rangle  + \eta_t^2 \E_{x} \norm{v_{\ell, h}(x)}^2.
 \end{align*}
 Using the identity from \equationref{eq:test_func_diff} and the definition of $v_{\ell, h}$ from \equationref{eq:gradient_def}, this is equal to:
 \begin{align*} 
 & -2 \eta_t \mathsf{err}_t + \eta_t^2 \E_{x} [(\ell(x, h(x),1) - \ell(x,h(x),0))^2].
\end{align*}
Because losses lie in $[0,\lmax]$, the second term is less than $\lmax^2\eta_t^2$. 
Furthermore, from the auditing guarantee, $|\mathsf{err}_t| > \epsilon$. By setting the step size $\eta^{(t)}$ to be $-\eps \cdot \sign(\mathsf{err}_t) \; / \; \lmax$, we conclude that the difference in potentials across adjacent time steps satisfies,
\begin{align*}
2 \eta_t \mathsf{err}_t + \eta_t^2 \E_{x} [(\ell(x, h(x),1) - \ell(x,h(x),0))^2] \leq 	-2\eta_t \mathsf{err}_t + \eta_t^2 \lmax^2 \leq -\eps^2 /\lmax^2.
\end{align*}
Since the potential is nonnegative and bounded above by $|\cYh|$, the maximum number of iterations until the termination criterion is met must be at most $|\cYh| \lmax^2 / \eps^2$.
\end{proof}

An important consequence of this result is that it reveals the existence of omnipredictors $\pt$ that admit computationally efficient approximations. 
This result is subtle, even in light of previous work on OI-style boosting algorithms. Intuitively, the final $\pt$ is built out by stringing together copies of functions in $\cH$, $\cL$, and decision rules $f_{\ell,t}$.
Because these decision rules $f_{\ell,t}$ are defined in terms of an optimization procedure involving the intermediate constructions $p^{(t)}$, which themselves depend on previous models $p^{(t-1)}$, a naive implementation of $\pt$ can result in a recursion that induces an exponentially large circuit.
Specifically, the naive implementation would make $\card{\Ych}$ recursive calls to the prior circuit in order to compute $f_{\ell,t}$, resulting in a growth rate of $\card{\Ych}^t$.

However, by carefully ordering the relevant computations and 
``caching'' previous work, we avoid this blow-up.   
The key insight is the following.
By designing a circuit that computes the value of $\pt(x,\yhat)$ for every $\yhat \in \Ych$ simultaneously, we can avoid recursive calls to the circuit.
By maintaining the intermediate computations of $q^{(t)}$, we can avoid a branching factor in the program and preserve efficiency.




\begin{theorem}
\theoremlabel{theorem:circuit_size}
Assume that the functions in $\cH$ and $\cL$ are computable by circuits of size at most $s$, then the predictor $\pt$  returned by the POI-Boost algorithm has size at most $\lmax^2 /\eps^2 \cdot \poly(s, |\cYh|)$.
\end{theorem}
\begin{proof}
The final predictor consists of a summation of the initial prediction, followed by the update from each iteration.
We bound the growth of the circuit computing the predictor by induction.
Formally, let $S_t$ be the circuit size for computing $q^{(t)}$.
Then, we show that $S_{t+1}\leq S_t + \poly(|\cYh|, s)$ for all $t\geq 1$.
Thus, by the overall bound on the iteration complexity, the final predictor can be implemented using a circuit of size $S \le \lmax^2 /\eps^2 \cdot \poly(s, |\cYh|)$.
To begin, the initial constant predictor $q^{(1)}$ can be implemented using a circuit of size at most $S_1 = \card{\Ych} \le  \poly(s,\card{\Ych})$ by hard-coding the constant vector.

By the update rule, each update incorporates a function of the form,
\begin{align}
\equationlabel{eq:intermediate_functions}
g^{(t)}(x) \defeq\eta^{(t)} v^{(t)}_{\ell, h}(x) =	 \eta^{(t)}\begin{bmatrix}
	 (\ell^{(t)}(x, \yhat_1, 1) - \ell^{(t)}(x, \yhat_1, 0))  \;\ind\{h^{(t)}(x) = \yhat_1 \}\\ 
	 \dots  \\
	 (\ell^{(t)}(x, \yhat_k, 1) - \ell^{(t)}(x, \yhat_k, 0)) \; \ind\{h^{(t)}(x) = \yhat_k \}
	 \end{bmatrix}\in \R^{|\cYh|},
\end{align}
where $\ell^{(t)}$ and $h^{(t)}$ define the test function surfaced by the auditing subroutine at time step $t$.
Within these updates, the function $h^{(t)}$ may \emph{a}) come from $\Hcal$ due to a POI violation or \emph{b}) equal $f_{\ell,t}$ for some $\ell \in \cL$ due to a DOI violation.

In the first case where $h^{(t+1)} \in \cH$, $q^{(t+1)}(x)$ can be computed by evaluating evaluating $q^{(t)}(x)$, and then evaluating $h(x)$ and $\ell(x,\yhat, y)$, for every $\yhat$ and $y$.
By assumption, the latter operations require circuits of size at most $\poly(s, |\cYh|)$.
Paired with the inductive hypothesis, the resulting circuit size can be bounded as $S_{t+1} \le S_t + \poly(s,\card{\Ych}) \le (t+1) \cdot \poly(s,\card{\Ych})$.

For the second case, 
we recall the definition of $f_{\ell,t}(\cdot)$, we can express its computation as a minimization over $\Ych$ of expected losses that depend on $q^{(t)}(\cdot)[\yhat]$ for each $\yhat \in \Ych$.
\begin{align*}
	f_{\ell,t}(x) = \argmin_{\yhat \in \Ych}\ \set{\ell(x, \yhat,0) + \left(\ell(x, \yhat, 1) - \ell(x, \yhat, 0)\right) \cdot q^{(t)}(x)[\yhat]}.
\end{align*}
Importantly, to compute each term in the minimization, we only need to compute the vector $q^{(t)}(x)$ once.
The remaining terms, $\ell(x, \yhat, 0)$ and $\ell(x, \yhat, 1)$ (for every $\yhat$), can again be  computed by a circuit of size $\poly(|\cYh|, s)$.
Since the minimization itself can be done by linearly enumerating over $\cYh$, we again preserve the invariant that $S_{t+1} \leq S_t + \poly(|\cYh|, s)$.
\end{proof}

\subsection{Reducing Auditing to Supervised Learning}
\sectionlabel{subsec:csc}

\begin{figure}[t]
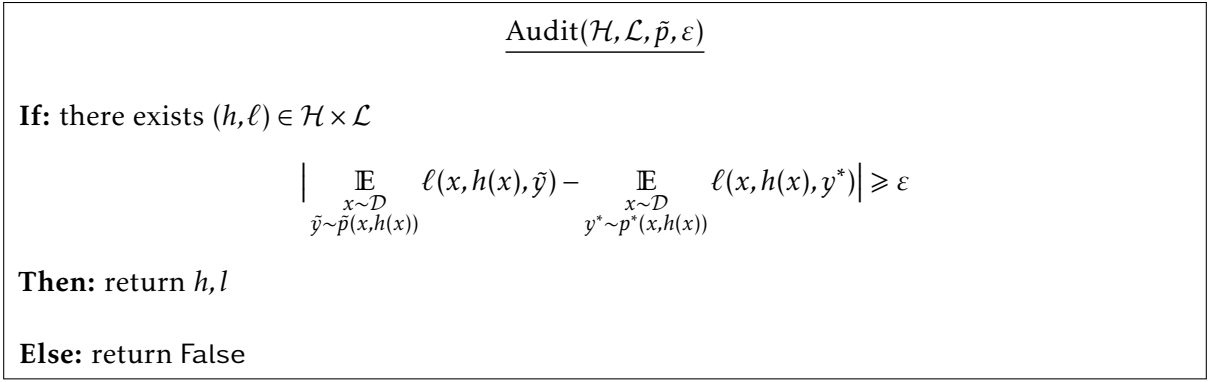

\setlength{\fboxsep}{2mm}
\begin{boxedminipage}{\textwidth}
\begin{center}
\underline{$\mathrm{Audit}(\cH, \cL, \pt, \eps)$}
\end{center}
\vspace{2mm}
\textbf{If:} there exists $(h,\ell) \in \cH \times \cL $
	\begin{align*}
		   \big|  \E_{\substack{x \sim \cD \\ \yt \sim \pt(x,h(x))}}\ell(x, h(x),\yt) - \E_{\substack{x \sim \cD \\ \ys \sim \ps(x,h(x))}}\ell(x, h(x),\ys) \big| \geq \eps
	\end{align*}
\textbf{Then:} return  $h,l$  \\

\textbf{Else:} return $\mathsf{False} $
\end{boxedminipage}
\caption{The key auditing step in the POI-Boost algorithm. In each iteration of the algorithm, we run two auditing steps: once to check for the POI condition over $\cH \times \cL$ and once to check for the DOI condition over $\{(f_{\ell, t}, \ell): \ell \in \cL \}$. See the proof of  ~\corollaryref{cor:end} for further discussion.}
\figurelabel{figure:auditing}
\end{figure}

Having shown how omniprediction reduces to an auditing problem, we now complete our analysis of the POI-Boost algorithm by showing that auditing itself reduces to cost-sensitive classification over a single, static distribution. In doing so, we address the statistical and computational complexity of solving this auditing step.


From examining the auditing condition in \figureref{figure:auditing}, perhaps the most obvious strategy is to choose a decision rule $h$, and to collect a dataset of triples $(x, \yhat, \ys)$ where $\yhat=h(x)$ for every $x$ and $\ys \sim \ps(x, h(x))$. If the loss $\ell$ is bounded, a standard application of Hoeffding's inequality shows that empirical risk of the loss concentrates around its expectation:
\begin{align*}
	 \frac{1}{n}\sum_{i=1}^n \ell(x_i, h(x_i), \ys_i) \approx\E_{\substack{x \sim \D\\ \ys \sim \ps(x,h(x))}} [\ell(x, h(x), y_*)].  
\end{align*} 
Therefore, one could implement the auditing step by enumerating over all $h$, deploying $h$ to collect a new dataset every time, and then nonadaptively computing the empirical performative risk of $h$ on every $\ell \in \cL$. This procedure would however require $\tilde{\cO}(|\cH|/\eps^2 \log |\cL|)$ many samples. 

On the other hand, if we have access to randomized predictions $\yhat$, we can estimate the empirical risk of every pair $h, \ell$ off of a \emph{single} distribution by using inverse propensity scoring. The following lemma is  well-known within various communities, and, in particular, the contextual bandits literature (see e.g. \cite{monster,dudik2011}).\footnote{This result can be generalized to the case where for every $x$, $\yhat \sim q(x)$ for some known distribution $q$, where $q\neq \Unif(\cYh)$ but where $q$ has full support over $\cYh$.}

We use the shorthand $(x,\yhat,y) \sim \D_\mathrm{rct}$ to denote the sampling process where inputs are sampled from the base distribution $x \sim \D$, decisions $\yhat$ are assigned uniformly at random, $\yhat\sim \Unif(\cYh)$, and the outcomes are sampled according to Nature's model $\ys \sim \ps(x,\yhat)$.

\begin{lemma}
\lemmalabel{lemma:ip_weighting}
Assume that $\calYhat$ is a finite set. Then, for any hypothesis $h:\cX \rightarrow \cYh$, 
\begin{align*}
\E_{\substack{x \sim \calD_x \\ \ys \sim \ps(x,h(x))}} [\ell(x, h(x), \ys) ]= |\calYhat| \cdot \E_{(x,\yhat, \ys) \sim \rct} [\ell(x, \yhat, \ys )\ind \{h(x) = \yhat\}].
\end{align*}	
\end{lemma}

\begin{proof}
We present the proof for the case where $\calYhat = \{0,1\}$ is binary, but the general case follows the same pattern. We expand out the left hand side as:
\begin{align*}
\E_{\substack{x \sim \calD \\ \ys \sim \ps(x,h(x))}} [\ell(x, h(x), \ys)] & = \E_{\substack{x \sim \calD \\ \ys_{(1)} \sim \ps(x,1), \ys_{(0)} \sim \ps(x,0)}} \left[\ell(x, 1, \ys_{(1)}) \ind\{h(x) = 1\} +  \ell(x, 0, \ys_{(0)}) \ind\{h(x) = 0\} \right]\\
& = \E_{\substack{x \sim \calD \\ \ys_{(1)} \sim \ps(x,1)}} \left[\ell(x, 1, \ys_{(1)}) \ind\{h(x) = 1\} \right] + \E_{\substack{x \sim \calD \\ \ys_{(0)} \sim \ps(x,0)}} \left[\ell(x, 1, \ys_{(0)}) \ind\{h(x) = 0\} \right].
\end{align*}
Reweighting the term on the right hand side, we observe our desired equality:
\begin{align*}
 \E_{\substack{x \sim \calD \\ \yhat \sim \Ber(1/2) \\ \ys \sim \ps(x,\yhat)}} \left[ \ell(x, \yhat, \ys )\ind\{h(x) = \yhat\} \right]&= \frac{1}{2}\E_{\substack{x \sim \calD \\ \ys_{(1)} \sim \ps(x,1)}} \left[\ell(x,1, \ys_{(1)}) \ind\{h(x) = 1\} \right] \\
 &+ \frac{1}{2} \E_{\substack{x \sim \calD \\ \ys_{(0)} \sim \ps(x,0)}} \left[\ell(x,1, \ys_{(0)}) \ind\{h(x) = 0\} \right]. 
\end{align*} 
\end{proof}

There are two main takeaways from this lemma. First, it shows that the statistical complexity of auditing can be exponentially better than the the naive strategy outlined previously.
\begin{corollary}
\corollarylabel{corollary:concentration}
Let $\{(x_i, \yhat_i, \yt_i)\}_{i=1}^n$ be a dataset of $n$ i.i.d samples from $\rct$. If $$n\geq \frac{2\lmax^2 |\cYh|^2 \cdot \log(2|\cH| |\cL| / \delta)}{\eps^2},$$ then with probability $1-\delta$,
\begin{align*}
\max_{h \in \cH,\ell \in \cL} \big|\E_{\substack{x \sim \calD \\ \ys \sim \ps(x,h(x))}} [\ell(x, h(x), \ys)]  - \frac{1}{n} \sum_{i=1}^n |\calYhat| \cdot  \ell(x_i, \yhat_i, \ys_i )\ind \{h(x_i) = \yhat_i\}  \big| \leq \eps.
\end{align*}
\end{corollary} 
 \begin{proof}
From the previous lemma, we have that for any loss $\ell$ and decision rule $h$, 
\begin{align*}
\E_{\substack{x \sim \calD \\ \ys \sim \ps(x,h(x))}} [\ell(x, h(x), \ys)] = |\calYhat|\cdot \E_{(x,\yhat, \ys) \sim \rct} [\ell(x, \yhat, \ys )\ind \{h(x) = \yhat\}].
\end{align*}	 
Because $|\calYhat| \cdot \ell(x_i, \yhat_i, \ys_i )\ind \{h(x_i) = \yhat_i\}$ is uniformly bounded by $\lmax |\cYh|$, we can apply Hoeffding's inequality to argue that the probability that the empirical estimate is far from the true expectation
\begin{align*}
\big|\E_{\substack{x \sim \calD \\ \ys \sim \ps(x,h(x))}} [\ell(x, h(x), \ys)]  - \frac{1}{n} \sum_{i=1}^n |\calYhat| \cdot  \ell(x_i, \yhat_i, \ys_i )\ind \{h(x_i) = \yhat_i\}  \big| > \epsilon
\end{align*}
is bounded as $2\exp\left(\frac{2\eps^2}{n (\lmax \card{\Ych})^2}\right)$.
The result follows by rearranging for failure probability $\delta$, and taking a union bound over all $h \in \cH$ and $\ell \in \cL$.
 \end{proof}

Consequently, for a single iteration of the POI-Boost algorithm, the auditing step can be implemented by enumerating over all $\cL$ and $\cH$ and non-adaptively evaluating their empirical risks on a single dataset of size $\tilde{\cO}(\lmax^2 |\cYh|^2 / \eps^2 \log(|\cH| |\cL| ))$.\footnote{An analogous result applies if we replace $\ps$ by $\pt$, which we assume that the learner can easily sample from.} 
Typically, we think of $|\cYh|$ as a small constant and the class of decision rules $\cH$ as a rich collection. From this result, we see that at least statistically, we can hope to design omnipredictors that are optimal with respect to an exponential number of losses and decision rules.

Here, we present the simplest possible analysis of this result and state our bounds for finite classes $\cH$ and $\cL$. It is certainly feasible to achieve sharper results and to state bounds in terms of VC-dimension or other sharper notions of statistical complexity. 
However, the goal of our initial work on outcome performativity is not to establish the tightest bounds, but to provide a broad overview of what is possible. 
We hope future work will provide a precise understanding of the sample complexity of omniprediction in outcome performativity. 

The following proposition summarizes the sample complexity of omniprediction if the auditing steps for the POI and DOI conditions are implemented via a naive learner that linearly enumerates over all $h,\ell$ and evaluates their empirical risk on a single dataset of RCT samples.

\begin{proposition}
\propositionlabel{prop:sc_summary}
Given labeled data $(x, \yhat, \ys) \sim \rct$ drawn from Nature and unlabeled samples $x \sim \cD$, the POI-boost can be implemented using at most: 
\begin{itemize}
	\item $\cO(\lmax^2 |\cYh|^2 \log(\frac{|\cH||\cL|}{\delta}) / \epsilon^2 + \lmax^4 |\cYh|^3 \log(\frac{|\cL| \lmax |\cYh|}{\delta \epsilon}) /\epsilon^4 )$ labeled samples
	\item $\cO(\lmax^4 |\cYh|^3 \log(\frac{|\cH||\cL| \lmax |\cYh|}{\delta \epsilon}) / \epsilon^4)$ unlabeled samples
\end{itemize}
\end{proposition}

\begin{proof}
In each iteration of the POI-boost algorithm, we need to audit for the POI and DOI guarantees (conditions $a$ and $b$). We can implement each of the auditing steps by explicit enumeration. 

For POI, at each iteration $t$ we enumerate over $\cH$ and $\cL$ and evaluate the empirical counterparts of 
\begin{align}
\equationlabel{eq:two_expectations}
\E_{\substack{x \sim \cD \\ \yt \sim p^{(t)}(x,h(x))}}[\ell(x, h(x),\yt)] \quad \text{ and } \E_{\substack{x \sim \cD \\ \ys \sim \ps(x,h(x))}}[\ell(x, h(x),\ys)].
\end{align}
By \corollaryref{corollary:concentration}, the empirical versions of these quantities concentrate around their expectations. To get an $\eps$ approximation, with probability $1-\delta$, we require at most $\cO(\lmax^2 |\cYh|^2 \log(|\cH||\cL| / \delta) / \epsilon^2)$ many samples. At each iteration $t$,  the expectation on the left changes, since we update $p^{(t)}$. However, to evaluate this expectation we only need \emph{unlabeled} samples, since labels $\yt$ come from our own model $p^{(t)}$. On the other hand, the expectation on the right in \equationref{eq:two_expectations} does not depend on $t$, so we need not recompute it at every iteration. Because the total number of iterations is bounded by $\lmax^2 |\cYh| / \epsilon^2$, applying a union bound on $\delta$, to achieve the POI guarantee we only need a total of $\cO(\lmax^4 |\cYh|^3 \log(|\cH||\cL| \lmax |\cYh| \epsilon^{-1} \delta^{-1}) / \epsilon^4)$ unlabeled samples and $\cO(\lmax^2 |\cYh| \log(|\cH||\cL| / \delta) / \epsilon^2)$ labeled samples.

For the DOI guarantee outline in condition $b$, we instead need to approximate 
\begin{align}
\equationlabel{eq:doi_two_equations}
\E_{\substack{x \sim \D\\\yt \sim p^{(t)}(x,f_{\ell,t}(x))}}[\ell(x,f_{\ell,t}(x),\yt)]	\quad \text{ and } \E_{\substack{x \sim \D\\\ys \sim \ps(x,f_{\ell,t}(x))}}[\ell(x,f_{\ell,t}(x),\ys)].
\end{align}
Note that both of these expectations now depend on $t$, because the decision rules $f_{\ell,t}$ can change between iterations. Again, by \corollaryref{corollary:concentration}, if we enumerate over all $|\cL|$ losses and decision rules $f_{\ell,t}$ at each iteration, the empirical counterparts of these expressions on a dataset of size $\cO(\lmax^2 |\cYh|^2 \log(|\cL|/\delta) /\epsilon^2)$ concentrates. Collecting a new dataset at every iteration, we get that the total number of labeled (and unlabeled) samples is bounded by $\cO(\lmax^4 |\cYh|^3 \log(|\cL| |\cYh| \lmax \delta^{-1} \epsilon^{-1}) /\epsilon^4)$.
\end{proof}

\paragraph{Cost-Sensitive Classification.} The second main takeaway from \lemmaref{lemma:ip_weighting} is that auditing can now be rewritten as as the solution to a cost-sensitive multiclass classification problem over $|\cYh|$ many classes. This result completes our analysis showing how omniprediction can be reduced to basic supervised learning problems. 

In light of previous results, the main benefit of this reduction is that it enabled the design of oracle-efficient algorithms which can be faster than the naive learner used in \propositionref{prop:sc_summary}. We start by first defining what we mean by cost-sensitive classification.

\begin{definition}
\definitionlabel{definition:csc}
Let $\cX$ be a feature space, $\cYh$ be a finite set of $k$ classes, and $\cD$ be a distribution over $\cX \times [-1,1]^k$. For $(x,c) \sim \cD$, we say that $c$ is a cost vector whose entries $c(\yhat)$ denote the costs of predicting label $\yhat$ on feature $x$. An algorithm $\cscalg$ is an $\rho$-cost-sensitive  learner for a hypothesis class $\cH$ if for any distribution $\cD$ over $\cX \times [-1,1]^{k}$, promised that there exists $h \in \cH$ such that $\E_{(x,c) \sim \cD} c(h(x)) \leq -\rho$, $\cscalg$ returns a hypothesis $h'$ such that $\E_{(x,c) \sim \cD} c(h'(x)) \leq -\rho/2 $.
\end{definition}

Cost-sensitive classification is a well-studied supervised learning problem for which many, both passive and active learning algorithms, have been designed \cite{beygelzimer2009error,abe2004iterative,langford2005sensitive, krishnamurthy2017active,elkan2001foundations}. 
There are a number of software packages that can be used to solve applied cost-senstive classification problems \cite{catboost}. 
Like many problems in computational learning theory, cost-sensitive classification is known to be hard in the worst-case, but can be solved effectively in practice.
As such, our goal is to design \emph{oracle-efficient} learning algorithms, that make an small number of calls to cost-sensitive learner.

Here, we frame a ``weak'' version of the problem where the learning need not be exact, but where the search is proper, in the sense that $\cscalg$ returns a hypothesis in $\cH$. This latter condition can easily also be relaxed without changing the overall results. However, we opt to keep it as is for the sake of simplifying the presentation.
The following proposition completes our reduction of auditing to supervised learning. 

\begin{proposition}
\proplabel{prop:csc_auditing}
Let $\cscalg$ be a cost-sensitive learner as per \definitionref{definition:csc}. Then, given access to RCT samples $(x, \yhat, \ys) \sim \rct$, we can solve the auditing problem outlined in \figureref{figure:auditing} using $2|\cL|$ many calls to $\cscalg$ with parameters $\rho  = \eps / (4\lmax |\cYh|)$.
\end{proposition}

\begin{proof}
By \lemmaref{lemma:ip_weighting} we have that the difference in performative risk between $\ps$ and $\pt$,
\begin{align*}
	  \E_{\substack{x \sim \cD \\ \yt \sim \pt(x,h(x))}}[\ell(x, h(x),\yt)] - \E_{\substack{x \sim \cD \\ \ys \sim \ps(x,h(x))}}[\ell(x, h(x),\ys)],
\end{align*}
is equal to:
\begin{align*}
|\cYh| \cdot \E_{\substack{x \sim \calD \\ \yhat \sim \Unif(\calYhat) \\ \yt \sim \pt(x, \yhat) ,\ys \sim \ps(x,\yhat)}} \left[\ind \{h(x) =\yhat\}(\ell(x, h(x),\yt) - \ell(x, h(x),\ys))\right].
\end{align*}
Now, we note that terms inside the expectation can be written as entries in a cost vector $c$ where for every sample $(x,\yhat, \ys, \yt)$ we define the corresponding vector $c$ to be, 
\begin{align*}
	c_{\sigma}(h(x)) = \begin{cases}
		\sigma\cdot (\ell(x, h(x),\yt) - \ell(x, h(x),\ys)) &\text{ if } h(x) = \yhat \\ 
		0 &\text{ o.w}
	\end{cases},
\end{align*}
Here, $\sigma \in \{\pm 1\}$
and we set $\sigma=1$ to get the desired equality. Hence, for a fixed loss $\ell$, we can transform RCT samples, $x\sim \cD, \yhat \sim \unif(\cYh), \ys \sim \ps(x,\yhat)$ to a cost sensitive classification problem such that for every $h \in \cH$, 
\begin{align*}
\E_{\substack{x \sim \cD'}} [c_{+1}(h(x))]  = |\cYh| \E_{\substack{x \sim \calD \\ \yhat \sim \Unif(\calYhat) \\ \yt \sim \pt(x, \yhat) ,\ys \sim \ps(x,\yhat)}} [\ind \{h(x) =\yhat\}(\ell(x, h(x),\yt) - \ell(x, h(x),\ys))].
\end{align*}
To solve the auditing problem outlined in \figureref{figure:auditing}, we need to check whether the absolute value of the difference is larger than $\epsilon$. To do this, it therefore to suffices to run $\cscalg$ twice (once with $\sigma=1$ and once with $\sigma=-1$) for every loss $\ell \in \cL$ to check if there exists a decision rule $h \in \cH$ such that:
\begin{align*}
	\E_{\substack{x \sim \cD'}} [c_{+1}(h(x))] \leq -\eps \text{ or } \E_{\substack{x \sim \cD'}} [c_{-1}(h(x))] \leq -\eps.
\end{align*}
Because we normalize the cost vectors to have entries in $[-1,1]$ in \definitionref{definition:csc}, we can scale the vectors $c_\sigma$ by $1/ (4|\cYh|\lmax)$ and divide the tolerance parameter $\eps$ by the corresponding amount to match the desired interface.
\end{proof}

\subsection{End-to-End Analysis}

Having now presented these reductions showing how omniprediction can be reduced to cost sensitive classification, we now summarize our results so far and establish end-to-end bounds on the runtime and sample complexities of achieving omniprediction. 

\begin{corollary}
\corollarylabel{cor:end}
Assume that $h \in \Hcal$ and $\ell \in \cL$ can be evaluated in time $\poly(\log(|\cH|))$ and $\poly(\log(|\cL|))$, respectively.
Let $\cscalg$ be a $\rho$-cost-sensitive weak learner for $\cH$ as per \definitionref{definition:csc}.
Assume that for any distribution $\cD_{\mathrm{csc}}$ over pairs $(x,c) \in \cX \times [-1,1]^k$, $\cscalg$ runs in time $\poly(\log(|\cH|), 1/\rho)$ and uses at most $\poly(\log(|\cH|), 1 / \rho)$ many samples drawn from $\cD_{\mathrm{csc}}$.\footnote{Here, we have avoided discussion on the failure probability parameter $\delta$ for the $\cscalg$. However, it is clear that the relevant complexity bounds should depend only on $\log(1/\delta)$ and that applying a simple union bound would not change the nature of the resulting analysis. We therefore assume that the algorithms succeed with probability 1 for the sake of simplicity.}
If the learner has access to samples drawn according $(x,\yhat, \ys) \sim \rct$, then, the POI-Boost algorithm:
\begin{itemize}
	\item runs in time $\mathcal{O}\left(|\cL| \cdot \poly\left(1/\eps,\; \lmax, \;|\cYh|, \;\log|\cL|,\;\log|\cH|\right) \right) $
	\item uses at most $\cO\left( \poly\left(1/\eps, \;\lmax, \;|\cYh|,\;\log|\cL|,\; \log|\cH|\right)\right)$ many samples
\end{itemize}
\end{corollary}
\begin{proof}
To bound the runtime, we note that by \propref{prop:iteration_bound}, the maximum number of iterations for POI-Boost is at most $\lmax^2 |\cYh| / \eps^2$. In each iteration, we solve the auditing via two subroutines. One to check for the POI guarantee (condition $a$ in \figureref{fig:algorithm}) and another routine to check the DOI guarantee (condition $b$ in \figureref{fig:algorithm}).
To audit for the POI guarantee, we call the cost-sensitive learner $2|\cL|$ times with parameters $\rho  = \eps / (4\lmax |\cYh|)$ as per \propref{prop:csc_auditing}, using labels derived from calculating each $\ell \in \cL$ and evaluating $p^{(t)}(x,\yhat)$ in at most $\poly(\log(\card{\cL})) + \poly(\log(\card{\cH}),1/\eps,\lmax,\card{\Ych})$ time. With the labels calculated, each of these calls has run time and sample complexity at most $\poly(\log(|\cH|, 1/\epsilon, \lmax, |\cYh|))$.

To audit for the DOI guarantee over the $(h,\ell) \in \{(f_{\ell,t}, \ell): \ell \in \cL\}$, at each iteration, we use the naive strategy outlined in \sectionref{subsec:csc} where we enumerate over all $|\cL|$ losses and evaluate the performative risk of each pair $(f_{\ell,t}, \ell)$ on a single dataset of RCT samples of size $\cO(\lmax^2  / \eps^2 |\cYh|^2 \log(|\cL|))$ as per \corollaryref{corollary:concentration}. Each auditing step for DOI therefore runs in time $|\cL| \cdot \poly(1/\epsilon, \lmax, |\cYh|)$ and uses $\poly(1/\epsilon, \lmax, |\cYh|, \log(|\cL|))$ many samples. All calls to the intermediate predictors $p^{(t)}$ also run in polynomial time as per \theoremref{theorem:circuit_size}. The final guarantees come from multiplying the sample and run time complexity of each iteration of the POI-boost algorithm by the bound on the total number of iterations.
\end{proof}

The main take away from this result is that if the cost-sensitive classification problem can be solved efficiently, in the sense that the the relevant statistical and computational complexities scale as $\polylog|\cH|$, then the overall POI-Boost algorithm runs in time linear in $|\cL|$, poly-logarithmically in the size of $\cH$ and with at most $\polylog |\cH| |\cL|$ many samples. Therefore, we can hope to develop efficient omnipredictors that are optimal for exponentially many decision rules, and polynomially many losses.

Note that, because of the result outlined in \propref{prop:adaptability},  this theorem also bounds the statistical and computational complexity of achieving universally adaptable omnipredictors. 
More specifically, the number of samples and the runtime for achieving universally adaptable omnipredictors are also bounded by the quantities in \corollaryref{cor:end} where we now replace the class $\cL$ by augment collection $\cL_\Wcal$ as defined in \propref{prop:adaptability}. The main difference is that the relevant runtime and sample complexity bounds replace dependence on $|\cL|$ by dependence on $|\cL||\cW|$ and replace dependence on $\lmax$ by $\lmax \omax$. Here, $\omax$ is the the worst case density ratio for the class $\cW$.
\begin{align*}
	   \max_{\omega \in \cW} \max_{x \in \supp(\D_\omega)}  \omega(x) = \frac{\D_\omega(x)}{\D(x)}.
\end{align*}
This complexity measure capture the intuition that if individuals $x$ are poorly represented over the distribution $\cD$ we are learning over, then we need more samples (and consequently runtime), to learn universally adaptable omnipredictors. 
We think of these complexity parameters like $\omax$ as a first step. It is an interesting question for future work to provide sharper notions of problem complexity and to find ways of designing omnipredictors for exponentially large collections of importance weights $\cW$.

%% file: calibration.tex
\section{Connections to Multicalibration}

So far, we have studied how  extensions of  outcome indistinguishability definitions enable the design of omnipredictors for performative settings. 
In the world of supervised learning, \cite{oi} established tight connections between outcome indistinguishability and various notions of multicalibration \cite{hkrr}.
Given the complementary relationship between these two concepts in the supervised world, it is natural to speculate that generalizing multicalibration to the outcome performative setting might be fruitful.

In this section, we begin to examine these questions and 
discuss analogues of multiaccuracy and multicalibration for the performative setting. 
We start by showing that multiaccuracy naturally, and efficiently, extends to performative contexts, and provides an effective way to achieve performative outcome indistinguishability in a loss-independent fashion. 
Conversely, we illustrate how naive translations of multicalibration to performative prediction result in definitions whose complexity blows up exponentially in the number of predictions $\card{\Ych}$.
We conclude with some discussion of alternatives to calibration-style guarantees that could, in principle, be used to obtain efficient omnipredictors.

\paragraph{On Multiaccuracy.}
\theoremref{thm:omni} shows how $(\cL,\Hcal, \epsilon)$-performative omniprediction arises as a consequence of $(\cL, \Hcal, \epsilon)$-performative OI and $(\cL, \epsilon)$-decision OI, where the OI distinguishers explicitly account for the collection of loss functions $\cL$.
Here, we show an efficient approach for obtaining POI for the class of all bounded input-oblivious loss functions.

We say that a loss function is \emph{input-oblivious} if it only depends on the input $x \in \Xcal$ via the decision $h(x)$. That is, for all $x$ and $x'$ and pairs $(\yhat, y)$, $\ell(x,\yhat, y) = \ell(x', \yhat,y)$. Equivalently, these functions have domain $\cYh \times \cY$ instead of $\cX\times \cYh \times \cY$.
We use 
\begin{align*}
\mathcal{L}_\mathrm{io} = \set{\ell:\Ych \times \set{0,1} \to [0,1]}
\end{align*} 
to denote the class of all bounded input-oblivious loss functions.
Our first result for this section proves that a performative analogue of multiaccuracy \cite{hkrr,kim2019multiaccuracy}  implies POI for $\cL_\mathrm{io}$.
\begin{definition}[Multiaccuracy]
For a distribution $\D$, hypothesis class $\Hcal$, and $\eps \ge 0$,
a predictor $\pt:\Xcal \times \Ych \to [0,1]$ is $(\Hcal,\eps)$-multiaccurate under outcome performativity over $\D$ if for all $h \in \Hcal$ and $\yhat \in \Ych$
\begin{gather*}
    \big|\E_{\substack{x \sim \calD \\ \ys \sim \ps(x,\yhat)}} [\ys \cdot \1\{h(x) = \yhat\}] - \E_{\substack{x \sim \calD \\ \yt \sim \pt(x,\yhat)}}[\yt \cdot \1\{h(x) = \yhat\}] \big|
  \leq \eps.
\end{gather*}
\end{definition}
Here, we require that the expectation of our modeled outcome $\yt \sim \pt(x,h(x))$ is accurate after deploying each $h \in \Hcal$, even when restricting our attention to the individuals $x \in \Xcal$ such that $h(x) = \yhat$.
While seemingly simpler than performative OI, we show that multiaccuracy in fact implies performative OI for all input-oblivious losses.
\begin{lemma}
\label{lem:MA2POI}
If $\pt:\Xcal \times \Ych \to [0,1]$ is $(\Hcal,\eps)$-multiaccurate, then $\pt$ is $(\cL_\mathrm{io},\Hcal,2\eps)$-performative OI.
\end{lemma}
\begin{proof}
The proof shows an approximate equality between the loss $\ell \in \mathcal{L}_\mathrm{io}$ of any $h \in \Hcal$ under $\yt \sim \pt(x,h(x))$ and $\ys \sim \ps(x,h(x))$. For $\cY = \{0,1\}$,
\begin{align*}
\E_{\substack{x \sim \calD \\ \yt \sim \pt(x,\yhat)}}[\ell(h(x),\yt)]
&= \sum_{\yhat \in \Ych} \Pr_{\cD}[h(x) = \yhat] \cdot \E_{\substack{x \sim \calD \\ \yt \sim \pt(x,\yhat)}}[\ell(\yhat,\yt) ~\vert~ h(x) = \yhat]\\
&= \sum_{\yhat \in \Ych} \Pr_{\cD}[h(x) = \yhat] \cdot \E_{\substack{x \sim \calD \\ \yt \sim \pt(x,\yhat)}}\left[\yt \cdot \ell(\yhat,1) + (1-\yt) \cdot \ell(\yhat,0) ~\vert~ h(x) = \yhat\right]\\
&\leq \sum_{\yhat \in \Ych} \Pr_{\cD}[h(x) = \yhat] \cdot \E_{\substack{x \sim \calD \\ \ys \sim \ps(x,\yhat)}}\left[\ys \cdot \ell(\yhat,1) + (1-\ys) \cdot \ell(\yhat,0) ~\vert~ h(x) = \yhat\right] + 2 \eps\\
&= \sum_{\yhat \in \Ych} \Pr_{\cD}[h(x) = \yhat] \cdot \E_{\substack{x \sim \calD \\ \ys \sim \ps(x,\yhat)}}[\ell(\yhat,\ys) ~\vert~ h(x) = \yhat] + 2 \eps\\
&=\E_{\substack{x \sim \calD \\ \ys \sim \ps(x,\yhat)}} [\ell(h(x),\ys)] + 2 \eps.
\end{align*}
The the third line follows under the assumption that $\pt$ is $(\Hcal,\eps)$-multiaccurate and the bound on the magnitude of $\card{\ell(\yhat,b)} \le 1$ for all $\yhat \in \Ych$ and $b \in \set{0,1}$. Given that an identical argument can be used to show the opposite inequality, we conclude that $\pt$ is indeed POI.
\end{proof}
Inspecting the performative multiaccuracy condition, we can see that it is similarly possible to reduce the problem of auditing for multiaccuracy to supervised learning.
This auditing procedure can be viewed as a special case of the auditing step from Section~\ref{sec:learning} or as a generalization of previous auditing procedures from work on multiaccuracy in the supervised learning setting \cite{hkrr,kim2019multiaccuracy}. In more detail, the relevant auditing problem for performative multiaccuracy is to determine whether there exists an $h \in \cH$ and $\yhat \in \cYh$ such that
\begin{gather*}
         \big|\E_{\substack{x\sim \cD \\ \ys \sim \ps(x,\yhat)}}[(\ys - \pt(x,\yhat)) \cdot \1\{h(x) = \yhat\}] \big|> \eps.
\end{gather*}
As before, this auditing step reduces to a cost-sensitive classification problem (\definitionref{definition:csc}). Auditing over a hypothesis class $\cH$ can be done with $2|\cYh|$ many calls to a cost-sensitive learner $\cscalg$ with tolerance parameter $\cO(\eps)$. We omit a formal statement of this result since it follows the exact pattern from \propref{prop:csc_auditing}.

In other words, in order to achieve $(\cL,\Hcal,O(\eps))$-POI for any class of input-oblivious losses $\cL \subseteq \cL_\mathrm{io}$, it suffices to audit for and enforce $(\Hcal,\eps)$-multiaccuracy.
In this sense, for input-oblivious losses, there is a single auditing procedure that works for all losses, so we can replace $\card{\cL}$-factors by $|\cYh|$ factors in the auditing complexity for performative OI. 

\paragraph{On Multicalibration.}
\label{sec:calibration}

Going beyond multiaccuracy, the original work of \cite{omni} established omniprediction in the supervised setting as a consequence of multicalibration.
As such, we might wonder whether there exists an analogous notion of multicalibration for performative prediction that enables a similar result.
Defining an efficient notion of calibration under performativity (let alone, multicalibration), turns out to be a subtle task.

In supervised learning, calibration requires that the expectation of $\pt$ is accurate, even when we partition the inputs $x \in \Xcal$ based on the predicted value $\pt(x) = v$.
Specifically, the constraints quantify over each supported $v \in \supp(\pt)$:
\begin{gather*}
    \E_{\substack{x \sim \D\\\ys \sim \ps(x)}}[\ys \cdot \1\{\pt(x) = v\}] \approx \E_{\substack{x \sim \D\\\yt \sim \pt(x)}}[\yt \cdot \1\{\pt(x) = v\}] = v \cdot \Pr[\pt(x) = v].
\end{gather*}
Such calibration-style constraints suffice to establish omniprediction because for any loss $\ell$, the optimal decision $\ft_\ell(x)$ is completely determined by $\pt(x)$.

In performative prediction, quantifying over the supported values of $\pt$ requires considering the decisions $\yhat \in \Ych$ as well.
In particular, the optimal decision $\ft_\ell(x)$ is a function of the vector of predictions $\tilde{q}(x) \in [0,1]^{\card{\Ych}}$, which gives the predicted probability $\yt \sim \pt(x,\yhat)$ for each $\yhat \in \Ych$ (recall the $q(\cdot)$ notation from \propref{prop:iteration_bound}).
Thus, using the naive translation of the calibration constraints for omniprediction, we must partition $\Xcal$ based on the vector-valued predictions,
$\tilde{q}(x) = \overrightarrow{v}$.
The cardinality of this calibration partition of $\Xcal$ scales exponentially in the number of decisions $\card{\Ych}$, even for binary outcomes $\cY$.

Still, we may consider more efficient calibration-style conditions that suffice to imply omniprediction in the performative setting.
Rather than aiming for full performative calibration, we focus on adapting the notion of decision calibration \cite{zhao2021calibrating} to the performative setting.
Decision calibration was introduced to avoid  exponential blow-up in the calibration constraints due to multi-class prediction.
We show that the notion can equally be adapted to the performative setting to deal with blow-up due to many actions $\yhat \in \Ych$.
We define decision calibration with respect to the class of input-oblivious losses.
\begin{definition}[Decision Calibration]
For a distribution $\D$ and $\eps \ge 0$,
a predictor $\pt:\Xcal \times \Ych \to [0,1]$ is $\eps$-decision calibrated under outcome performativity over $\D$ if for every loss $\ell \in \cL_\mathrm{io}$, and for all $\yhat \in \Ych$,
\begin{gather*}
   \big| \E_{\substack{x \sim \D\\\ys \sim \ps(x,\ft_\ell(x))}}[\ys \cdot \1\{\ft_\ell(x) = \yhat\}] - \E_{\substack{x \sim \D\\\yt \sim \pt(x,\ft_\ell(x))}}[\yt \cdot \1\{\ft_\ell(x) = \yhat\}] \big| \leq \eps.
\end{gather*}
\end{definition}
With this definition in place, the proof of Lemma~\ref{lem:MA2POI} can be adapted to show that decision calibration suffices to establish performative decision OI.
\begin{lemma}
\label{lem:DC2DOI}
If $\pt:\Xcal \times \Ych \to [0,1]$ is $\eps$-performative decision calibrated, then $\pt$ is $(\cL_\mathrm{io},2\eps)$-performative decision OI.
\end{lemma}
As an immediate corollary of ~\theoremref{thm:omni} and Lemmas~\ref{lem:MA2POI} \& \ref{lem:DC2DOI}, we obtain sufficient conditions for omniprediction with respect to all bounded input-oblivious losses.
\begin{corollary}
Suppose $\pt:\Xcal \times \Ych \to [0,1]$ is $(\Hcal,\eps)$-multiaccurate and $\eps$-decision calibrated under outcome performativity.
Then, $\pt$ is a $(\cL_\mathrm{io},\Hcal,4\eps)$-performative omnipredictor.
\end{corollary}
In other words, if we can obtain multiaccuracy and decision calibration under performativity, then we have a direct pathway to obtain omniprediction for all bounded, input-oblivious losses.

\paragraph{On Decision Calibration.}
Note, however, that unlike the case of multiaccuracy, the decision rules that arise in the decision calibration condition are loss-dependent.
That is, the optimal decision rules $\ft_\ell$ depend on $\ell \in \cL_\mathrm{io}$.

Motivated by the strong guarantee, we consider the feasibility of auditing for decision calibration.
Note that for any $\ell \in \cL_\mathrm{io}$, the loss is defined by the loss for each outcome $y \in \set{0,1}$ and decision $\yhat \in \Ych$.
Thus, to audit decision calibration over all input-oblivious losses, it suffices to audit whether there exist a $\yhat \in \Ych$ and $w_{a,0},w_{a,1} \in [-1,1]$ such that:
\begin{gather*}
    \big|\E_x[\left(\pt(x,\yhat) - \ys\right) \cdot \1\{\argmin_{a \in \Ych}\set{ w_{a,1} \cdot \pt(x,a) + w_{a,0} \cdot (1-\pt(x,a))} = \yhat\}] \big| > \eps
\end{gather*}
where the weights $w_{a,0}$ and $w_{a,1} \in [-1,1]$ represent the choice of $\ell(a,0)$ and $\ell(a,1)$ corresponding to the choice of $a \in \Ych$.

Naively, searching for such a violated loss might require time exponential in $\card{\Ych}$.
For instance, by explicitly enumerating over some appropriately-fine net of $[-1,1]^{2\card{\Ych}}$, then we can simply consider ``every'' possible loss.
Improving the computational complexity of such a search is an interesting question, which may benefit from the techniques utilized in \cite{zhao2021calibrating}.

Even without an improvement in the runtime complexity of learning, note that once an auditor succeeds, and we have a violated loss function, there is an efficient update to the prediction function.
In particular, we simply need to record the chosen $2 \cdot \card{\Ych}$ parameters $\set{w_{a,0},w_{a,1}}$ and the $\yhat \in \Ych$, then execute the update from POI-Boost.
In all, we can conclude that performative omnipredictors for $\cL_\mathrm{io}$ exist in complexity independent of the complexity of $\cL_\mathrm{io}$.
\begin{corollary}
Suppose $\Hcal$ is a hypothesis class with size-$s$ circuits.
Then, for any $\eps > 0$, there exist an $(\cL_\mathrm{io},\Hcal,\eps)$-performative omnipredictor implemented by a circuit of size  $\poly(s, \card{\Ych})/\eps^2$.
\end{corollary}
This preliminary analysis leaves open the possibility of learning performative omnipredictors via techniques that are independent of the loss class, as in the original work on $(\cL_\mathrm{cvx},\Hcal)$-omniprediction from $\Hcal$-multicalibration.
We leave a more thorough investigation of these ideas to future work.

%% file: main.bbl
\newcommand{\etalchar}[1]{$^{#1}$}
\begin{thebibliography}{HMPW16}

\bibitem[AHK{\etalchar{+}}14]{monster}
Alekh Agarwal, Daniel Hsu, Satyen Kale, John Langford, Lihong Li, and Robert
  Schapire.
\newblock Taming the monster: A fast and simple algorithm for contextual
  bandits.
\newblock In {\em International Conference on Machine Learning}, pages
  1638--1646. PMLR, 2014.

\bibitem[AZL04]{abe2004iterative}
Naoki Abe, Bianca Zadrozny, and John Langford.
\newblock An iterative method for multi-class cost-sensitive learning.
\newblock In {\em Proceedings of the tenth ACM SIGKDD international conference
  on Knowledge discovery and data mining}, pages 3--11, 2004.

\bibitem[BB19]{balfanz2019early}
Robert Balfanz and Vaughan Byrnes.
\newblock Early warning indicators and intervention systems: State of the
  field.
\newblock {\em Handbook of student engagement interventions}, pages 45--55,
  2019.

\bibitem[BHK22]{brown2022performative}
Gavin Brown, Shlomi Hod, and Iden Kalemaj.
\newblock Performative prediction in a stateful world.
\newblock In {\em International Conference on Artificial Intelligence and
  Statistics}, pages 6045--6061. PMLR, 2022.

\bibitem[BLR09]{beygelzimer2009error}
Alina Beygelzimer, John Langford, and Pradeep Ravikumar.
\newblock Error-correcting tournaments.
\newblock In {\em International Conference on Algorithmic Learning Theory},
  pages 247--262. Springer, 2009.

\bibitem[CDH21]{cutler2021stochastic}
Joshua Cutler, Dmitriy Drusvyatskiy, and Zaid Harchaoui.
\newblock Stochastic optimization under distributional drift.
\newblock {\em arXiv preprint arXiv:2108.07356}, 2021.

\bibitem[DHK{\etalchar{+}}11]{dudik2011}
Miroslav Dudik, Daniel Hsu, Satyen Kale, Nikos Karampatziakis, John Langford,
  Lev Reyzin, and Tong Zhang.
\newblock Efficient optimal learning for contextual bandits.
\newblock In {\em Proceedings of the Twenty-Seventh Conference on Uncertainty
  in Artificial Intelligence}, UAI'11, page 169–178, Arlington, Virginia,
  USA, 2011. AUAI Press.

\bibitem[DKR{\etalchar{+}}21]{oi}
Cynthia Dwork, Michael~P. Kim, Omer Reingold, Guy~N. Rothblum, and Gal Yona.
\newblock Outcome indistinguishability.
\newblock In {\em ACM Symposium on Theory of Computing (STOC'21)}, 2021.

\bibitem[DKR{\etalchar{+}}22]{dwork2022beyond}
Cynthia Dwork, Michael~P. Kim, Omer Reingold, Guy~N. Rothblum, and Gal Yona.
\newblock Beyond bernoulli: Generating random outcomes that cannot be
  distinguished from nature.
\newblock In {\em International Conference on Algorithmic Learning Theory},
  pages 342--380. PMLR, 2022.

\bibitem[DX22]{drusvyatskiy2022stochastic}
Dmitriy Drusvyatskiy and Lin Xiao.
\newblock Stochastic optimization with decision-dependent distributions.
\newblock {\em Mathematics of Operations Research}, 2022.

\bibitem[Elk01]{elkan2001foundations}
Charles Elkan.
\newblock The foundations of cost-sensitive learning.
\newblock In {\em International joint conference on artificial intelligence},
  volume~17, pages 973--978. Lawrence Erlbaum Associates Ltd, 2001.

\bibitem[GHK{\etalchar{+}}23]{gopalan2022loss}
Parikshit Gopalan, Lunjia Hu, Michael~P. Kim, Omer Reingold, and Udi Wieder.
\newblock Loss minimization through the lens of outcome indistinguishability.
\newblock In {\em ITCS}, 2023.

\bibitem[GKR{\etalchar{+}}22]{omni}
Parikshit Gopalan, Adam~Tauman Kalai, Omer Reingold, Vatsal Sharan, and Udi
  Wieder.
\newblock Omnipredictors.
\newblock In {\em ITCS}, 2022.

\bibitem[HJM22]{hardt2022performative}
Moritz Hardt, Meena Jagadeesan, and Celestine {Mendler-D{\"u}nner}.
\newblock Performative power.
\newblock {\em arXiv preprint arXiv:2203.17232}, 2022.

\bibitem[HKRR18]{hkrr}
Ursula {H{\'e}bert-Johnson}, Michael~P. Kim, Omer Reingold, and Guy~N.
  Rothblum.
\newblock Multicalibration: Calibration for the (computationally-identifiable)
  masses.
\newblock In {\em International Conference on Machine Learning}, pages
  1939--1948. PMLR, 2018.

\bibitem[HMPW16]{hardt2016strategic}
Moritz Hardt, Nimrod Megiddo, Christos Papadimitriou, and Mary Wootters.
\newblock Strategic classification.
\newblock In {\em Proceedings of the 2016 ACM conference on innovations in
  theoretical computer science}, pages 111--122, 2016.

\bibitem[IYZ21]{izzo2021learn}
Zachary Izzo, Lexing Ying, and James Zou.
\newblock How to learn when data reacts to your model: performative gradient
  descent.
\newblock In {\em International Conference on Machine Learning}, pages
  4641--4650. PMLR, 2021.

\bibitem[IZY22]{izzo2022learn}
Zachary Izzo, James Zou, and Lexing Ying.
\newblock How to learn when data gradually reacts to your model.
\newblock In {\em International Conference on Artificial Intelligence and
  Statistics}, pages 3998--4035. PMLR, 2022.

\bibitem[JLP{\etalchar{+}}21]{jung2021moment}
Christopher Jung, Changhwa Lee, Mallesh Pai, Aaron Roth, and Rakesh Vohra.
\newblock Moment multicalibration for uncertainty estimation.
\newblock In {\em Conference on Learning Theory}, pages 2634--2678. PMLR, 2021.

\bibitem[JZM22]{jagadeesan22a}
Meena Jagadeesan, Tijana Zrnic, and Celestine {Mendler-D{\"u}nner}.
\newblock Regret minimization with performative feedback.
\newblock In Kamalika Chaudhuri, Stefanie Jegelka, Le~Song, Csaba Szepesvari,
  Gang Niu, and Sivan Sabato, editors, {\em Proceedings of the 39th
  International Conference on Machine Learning}, volume 162 of {\em Proceedings
  of Machine Learning Research}, pages 9760--9785. PMLR, 17--23 Jul 2022.

\bibitem[KAH{\etalchar{+}}17]{krishnamurthy2017active}
Akshay Krishnamurthy, Alekh Agarwal, Tzu-Kuo Huang, Hal Daum{\'e}~III, and John
  Langford.
\newblock Active learning for cost-sensitive classification.
\newblock In {\em International Conference on Machine Learning}, pages
  1915--1924. PMLR, 2017.

\bibitem[KGZ19]{kim2019multiaccuracy}
Michael~P. Kim, Amirata Ghorbani, and James Zou.
\newblock Multiaccuracy: Black-box post-processing for fairness in
  classification.
\newblock In {\em Proceedings of the 2019 AAAI/ACM Conference on AI, Ethics,
  and Society}, pages 247--254, 2019.

\bibitem[KKG{\etalchar{+}}22]{ua}
Michael~P. Kim, Christoph Kern, Shafi Goldwasser, Frauke Kreuter, and Omer
  Reingold.
\newblock Universal adaptability: Target-independent inference that competes
  with propensity scoring.
\newblock {\em Proceedings of the National Academy of Sciences},
  119(4):e2108097119, 2022.

\bibitem[KLMO15]{kleinberg2015prediction}
Jon Kleinberg, Jens Ludwig, Sendhil Mullainathan, and Ziad Obermeyer.
\newblock Prediction policy problems.
\newblock {\em American Economic Review}, 105(5):491--95, 2015.

\bibitem[KNRW18]{kearns2018preventing}
Michael Kearns, Seth Neel, Aaron Roth, and Zhiwei~Steven Wu.
\newblock Preventing fairness gerrymandering: Auditing and learning for
  subgroup fairness.
\newblock In {\em International Conference on Machine Learning}, pages
  2564--2572. PMLR, 2018.

\bibitem[KRR18]{kim2018fairness}
Michael~P. Kim, Omer Reingold, and Guy~N. Rothblum.
\newblock Fairness through computationally-bounded awareness.
\newblock {\em Advances in Neural Information Processing Systems}, 31, 2018.

\bibitem[LB05]{langford2005sensitive}
John Langford and Alina Beygelzimer.
\newblock Sensitive error correcting output codes.
\newblock In {\em International Conference on Computational Learning Theory},
  pages 158--172. Springer, 2005.

\bibitem[MBBF99]{mason1999boosting}
Llew Mason, Jonathan Baxter, Peter Bartlett, and Marcus Frean.
\newblock Boosting algorithms as gradient descent.
\newblock {\em Advances in Neural Information Processing Systems}, 12, 1999.

\bibitem[MDW22]{mendler2022predicting}
Celestine {Mendler-D{\"u}nner}, Frances Ding, and Yixin Wang.
\newblock Predicting from predictions.
\newblock {\em arXiv preprint arXiv:2208.07331}, 2022.

\bibitem[MMH20]{miller2020strategic}
John Miller, Smitha Milli, and Moritz Hardt.
\newblock Strategic classification is causal modeling in disguise.
\newblock In {\em International Conference on Machine Learning}, pages
  6917--6926. PMLR, 2020.

\bibitem[MPZ21]{miller2021outside}
John~P Miller, Juan~C Perdomo, and Tijana Zrnic.
\newblock Outside the echo chamber: Optimizing the performative risk.
\newblock In {\em International Conference on Machine Learning}, pages
  7710--7720. PMLR, 2021.

\bibitem[MPZH20]{mendler2020stochastic}
Celestine {Mendler-D{\"u}nner}, Juan Perdomo, Tijana Zrnic, and Moritz Hardt.
\newblock Stochastic optimization for performative prediction.
\newblock {\em Advances in Neural Information Processing Systems},
  33:4929--4939, 2020.

\bibitem[NFD{\etalchar{+}}22]{narang2022multiplayer}
Adhyyan Narang, Evan Faulkner, Dmitriy Drusvyatskiy, Maryam Fazel, and
  Lillian~J Ratliff.
\newblock Multiplayer performative prediction: Learning in decision-dependent
  games.
\newblock {\em arXiv preprint arXiv:2201.03398}, 2022.

\bibitem[PGV{\etalchar{+}}18]{catboost}
Liudmila Prokhorenkova, Gleb Gusev, Aleksandr Vorobev, Anna~Veronika Dorogush,
  and Andrey Gulin.
\newblock Catboost: unbiased boosting with categorical features.
\newblock {\em Advances in neural information processing systems}, 31, 2018.

\bibitem[PZMH20]{performative}
Juan Perdomo, Tijana Zrnic, Celestine {Mendler-D{\"u}nner}, and Moritz Hardt.
\newblock Performative prediction.
\newblock In {\em International Conference on Machine Learning}, pages
  7599--7609. PMLR, 2020.

\bibitem[SB14]{shalev2014understanding}
Shai {Shalev-Shwartz} and Shai {Ben-David}.
\newblock {\em Understanding machine learning: From theory to algorithms}.
\newblock Cambridge university press, 2014.

\bibitem[{US }16]{survey}
{US Department of Education}.
\newblock Issue brief: Early warning systems, 2016.
\newblock Available at
  \url{https://www2.ed.gov/rschstat/eval/high-school/early-warning-systems-brief.pdf}.

\bibitem[ZKS{\etalchar{+}}21]{zhao2021calibrating}
Shengjia Zhao, Michael~P. Kim, Roshni Sahoo, Tengyu Ma, and Stefano Ermon.
\newblock Calibrating predictions to decisions: A novel approach to multi-class
  calibration.
\newblock {\em Advances in Neural Information Processing Systems},
  34:22313--22324, 2021.

\end{thebibliography}
